\documentclass{article}

\PassOptionsToPackage{semicolon, round}{natbib}

\usepackage[main, final]{neurips_2026}

\usepackage[utf8]{inputenc}
\usepackage[T1]{fontenc}
\usepackage{hyperref}
\usepackage{url}
\usepackage{booktabs}
\usepackage{amsfonts}
\usepackage{nicefrac}
\usepackage{microtype}
\usepackage{xcolor}
\usepackage[linesnumbered,ruled,vlined]{algorithm2e}

\usepackage[utf8]{inputenc}
\usepackage[T1]{fontenc}
\usepackage{hyperref}
\usepackage{url}
\usepackage{booktabs}
\usepackage{amsfonts}
\usepackage{nicefrac}
\usepackage{microtype}
\usepackage{enumitem}
\usepackage{float}
\usepackage{blkarray}
\usepackage{caption}
\usepackage{multirow}
\usepackage{multicol}
\usepackage{listings}

\lstdefinestyle{mystyle}{
    backgroundcolor=\color{backcolour},   
    commentstyle=\color{codegreen},
    keywordstyle=\color{magenta},
    numberstyle=\tiny\color{codegray},
    stringstyle=\color{codepurple},
    basicstyle=\ttfamily\footnotesize,
    breakatwhitespace=false,         
    breaklines=true,                 
    captionpos=b,                    
    keepspaces=true,                 
    numbers=left,                    
    numbersep=5pt,                  
    showspaces=false,                
    showstringspaces=false,
    showtabs=false,                  
    tabsize=2
}
\usepackage{wrapfig}

\definecolor{darkblue}{rgb}{0.0,0.0,0.65}
\definecolor{darkred}{rgb}{0.65,0.0,0.0}
\definecolor{darkgreen}{rgb}{0.0,0.5,0.0}
\definecolor{tab:blue}{RGB}{31,119,180}
\definecolor{tab:red}{RGB}{214,39,40}
\definecolor{tab:green}{RGB}{44,160,44}
\definecolor{tab:orange}{RGB}{255,127,14}

\hypersetup{
	colorlinks = true,
	citecolor  = darkblue,
	linkcolor  = darkred,
	filecolor  = darkblue,
	urlcolor   = darkblue,
}

\definecolor{codegreen}{rgb}{0,0.6,0}
\definecolor{codegray}{rgb}{0.5,0.5,0.5}
\definecolor{codepurple}{rgb}{0.58,0,0.82}
\definecolor{backcolour}{rgb}{0.95,0.95,0.92}

\usepackage{amsmath}
\usepackage{amssymb}
\usepackage{mathtools}
\usepackage{amsthm}
\usepackage{thmtools, thm-restate}
\usepackage[capitalize,noabbrev]{cleveref}

\usepackage{graphicx}
\usepackage{subcaption}

\theoremstyle{plain}
\newtheorem{theorem}{Theorem}[section]
\newtheorem{proposition}[theorem]{Proposition}
\newtheorem{lemma}[theorem]{Lemma}
\newtheorem{corollary}[theorem]{Corollary}

\theoremstyle{definition}
\newtheorem{definition}[theorem]{Definition}

\theoremstyle{remark}
\newtheorem{remark}[theorem]{Remark}

\def\vbeta{{\bm{\beta}}}
\def\vgamma{{\bm{\gamma}}}
\def\He{\mathrm{He}}
\def\ie{\mathrm{ie}}
\def\ge{\mathrm{ge}}
\def\deg{\mathrm{deg}}
\def\poly{\mathrm{poly}}
\def\polylog{\mathrm{polylog}}
\def\bigO{{\mathcal{O}}}
\def\bigOtilde{{\tilde{\mathcal{O}}}}
\def\Mamba{{\mathsf{Mamba}}}
\def\MLP{{\mathsf{MLP}}}
\def\relu{{\mathrm{ReLU}}}
\def\pt{{\mathrm{pt}}}

\usepackage{latexsym}
\usepackage{amsmath}
\usepackage{amssymb}
\usepackage{bm}
\usepackage{mathrsfs}
\usepackage{url}
\usepackage{bbm}
\usepackage{enumitem}
\usepackage{nicefrac}

\def\norm#1{\lVert#1\rVert}

\newcommand{\mycomment}[1]{}

\def\eqref#1{(\ref{#1})}

\def\abs#1{\left|#1\right|}
\def\norm#1{\left\|#1\right\|}

\def\rvx{{\mathbf{x}}}

\def\rvz{{\mathbf{z}}}

\def\vzero{{\bm{0}}}
\def\vone{{\bm{1}}}

\def\va{{\bm{a}}}

\def\vc{{\bm{c}}}

\def\vh{{\bm{h}}}

\def\vn{{\bm{n}}}
\def\vo{{\bm{o}}}

\def\vu{{\bm{u}}}
\def\vv{{\bm{v}}}
\def\vw{{\bm{w}}}
\def\vx{{\bm{x}}}

\def\vz{{\bm{z}}}

\def\mA{{\bm{A}}}
\def\mB{{\bm{B}}}
\def\mC{{\bm{C}}}

\def\mI{{\bm{I}}}

\def\mM{{\bm{M}}}

\def\mW{{\bm{W}}}

\def\mZ{{\bm{Z}}}

\DeclareMathAlphabet{\mathsfit}{\encodingdefault}{\sfdefault}{m}{sl}
\SetMathAlphabet{\mathsfit}{bold}{\encodingdefault}{\sfdefault}{bx}{n}

\def\gD{{\mathcal{D}}}

\def\gF{{\mathcal{F}}}

\def\gI{{\mathcal{I}}}

\def\gN{{\mathcal{N}}}

\def\gR{{\mathcal{R}}}

\def\gT{{\mathcal{T}}}

\newcommand{\R}{\mathbb{R}}

\newcommand{\sigmoid}{\sigma}

\DeclareMathOperator*{\argmin}{arg\,min}

\title{Mamba Can Learn Low-Dimensional Targets In-Context via Test-Time Feature Learning}

\author{
  Junsoo Oh\thanks{Based on work performed during internship at the University of Tokyo} \\
  KAIST AI\\
  \texttt{\small junsoo.oh@kaist.ac.kr} \\
  \and
  \textbf{Wei Huang} \\
  RIKEN AIP/The Institute of Statistical Mathematics \\
  \texttt{\small wei.huang.vr@riken.jp} \\
  \and
  \textbf{Taiji Suzuki} \\
  The University of Tokyo/RIKEN AIP \\
  \texttt{\small taiji@mist.i.u-tokyo.ac.jp} \\
}

\begin{document}

\maketitle

\begin{abstract}
  Mamba, a recently proposed linear-time sequence model, has attracted significant attention for its computational efficiency and strong empirical performance. However, a rigorous theoretical understanding of its underlying mechanisms remains limited. In this work, we provide a theoretical analysis of Mamba's in-context learning (ICL) capability by focusing on tasks defined by low-dimensional nonlinear target functions. Specifically, we study in-context learning of a single-index model $y \approx g_*(\langle \boldsymbol{\beta}, \boldsymbol{x} \rangle)$, which depends on only a single relevant direction $\vbeta$, referred to as \emph{feature}. We prove that Mamba, pretrained by gradient-based methods, can achieve efficient ICL via \emph{test-time feature learning}, extracting the relevant direction directly from context examples. Consequently, we establish a test-time sample complexity that improves upon linear Transformers---analyzed to behave like kernel methods---and is comparable to nonlinear Transformers, which have been shown to surpass the Correlational Statistical Query (CSQ) lower bound and achieve near information-theoretically optimal rate in previous works. Our analysis reveals the crucial role of the \emph{nonlinear gating} mechanism in Mamba for feature extraction, highlighting it as the fundamental driver behind Mamba’s ability to achieve both computational efficiency and high performance.
\end{abstract}

\section{Introduction}

Mamba \citep{gu2024mamba}, a recently proposed state space model, has rapidly gained attention for its remarkable balance of computational efficiency and empirical performance. By replacing the quadratic-time attention mechanism of Transformers~\citep{vaswani2017attention} with a selective state-space recurrence with nonlinear gating, Mamba enables scalable modeling of long sequences while maintaining competitive accuracy across a variety of tasks~ \citep{dao2024transformers,waleffe2024empirical,wang2024mamba,patro2025mamba}. 
Despite Mamba’s remarkable computational efficiency, it remains unknown whether it can exhibit strong adaptability (often referred to as \emph{feature learning}), a property widely recognized as critical to the success of deep learning~neural networks \citep{girshick2014rich,suzuki2018adaptivity,damian2022neural}.

A key benchmark for test-time adaptability is in-context learning (ICL) \citep{brown2020language}, which has emerged as a canonical paradigm for understanding the adaptability of large language models and sequence architectures. By conditioning on context examples provided in the input prompt, a model can achieve strong performance on new tasks at test time without explicit parameter updates. While the empirical effectiveness of ICL is well documented, theoretical understanding of when and how different architectures exhibit this behavior remains limited \citep{xie2021explanation,garg2022can,zhou2023algorithms}. In particular, most existing theoretical analyses focus on Transformers \citep{ahn2023transformers,zhang2024trained,mahankali2024one,huang2023context,kim2024nonlinear}, whose quadratic attention mechanisms make them both powerful and computationally demanding. It remains unclear whether alternative architectures such as Mamba can offer comparable adaptability.

Recent works have investigated Mamba’s ICL capabilities, empirically demonstrating that Mamba performs competitively across various ICL benchmarks \citep{grazzi2024mamba,park2024can,li2024can}. 
However, our understanding of Mamba’s ICL capabilities remains lacking. This is due to its distinct inductive bias from the Transformer. The recurrent state-space model with nonlinear gating processes inputs through recurrent updates that maintain and transform hidden states over time, rather than relying on global attention over the entire context. This distinction motivates new theoretical questions:
\begin{center}
    \emph{Can Mamba provably achieve strong test-time adaptability like Transformers\\ with its recurrent state-space updates and nonlinear gating?}
\end{center}
\subsection{Summary of Contributions}
In this paper, we study the ICL capabilities of Mamba, focusing on a single-index model---a widely adopted theoretical tool for studying adaptability. We summarize our contributions as follows:
\begin{itemize}[leftmargin=*]
    \item We introduce a theoretical framework for analyzing Mamba’s ICL of single-index models, including input embeddings, the Mamba architecture, and a gradient-based pretraining algorithm. Under this framework, we characterize the optimization dynamics and establish the sample complexity in terms of the number of pretraining tasks and the number of context examples at pretraining and test time required to achieve strong performance (Theorem~\ref{thm:main}).
    \item Our analysis reveals that pretrained Mamba is capable of \emph{test-time feature learning}, enabling it to extract task-relevant features directly from context examples (Proposition~\ref{proposition:test-time_feature_learning}). This result implies that Mamba can surpass the performance of kernel regression baselines and achieve adaptation at test time. Specifically, the gating mechanism enables Mamba to achieve test-time feature learning, thereby overcoming the limitations inherent to purely linear recurrent updates.
    \item We provide a comparative analysis between Mamba and Transformer, highlighting similarities and differences in their ICL mechanisms. Our results reveal that Mamba can achieve test-time feature learning via a qualitatively different mechanism—recurrent state-space updates with nonlinear gating—thus extending the theoretical landscape of ICL beyond attention-based models.
\end{itemize}
\subsection{Related Works}
\paragraph{Theory of In-Context Learning.}Theoretical investigations of ICL have predominantly centered on Transformers. Beyond initial results showing that Transformers trained on regression tasks can reproduce ordinary least squares solutions in-context \citep{aky2023what,zhang2024trained,mahankali2024one,han2023explaining}, subsequent analyses reveal their ability to emulate more complex procedures such as multi-step gradient descent \citep{ahn2023transformers,saunshi2025reasoning}, functional gradient descent \citep{cheng2024transformers}, and sparse regression \citep{bai2023transformers}. Parallel works extend this line of inquiry to classification, where recent studies provide provable insights into how Transformers implement in-context classification \citep{li2024training,bu2025provable}. 

While the theoretical literature on ICL has dominantly focused on Transformers, a growing body of work is extending this theoretical analysis to linear-time sequence models. Recent works~\citep{li2024fine, li2025gating} prove that H3-like model~\citep{fu2023hungry} and gated linear attention~\citep{yang2024gated} can implement weighted preconditioned gradient descent based on loss landscape analysis. \citet{bondaschi2025markov} study ICL of Mamba on Markov chains and show that it learns a Laplacian smoothing estimator in-context. However, these works do not provide optimization or generalization guarantees. Recent works by \citet{jiang2025trained} and \citet{li2025can} provide such guarantees for ICL of Mamba on linear regression and classification tasks with outliers, as we also do in this work.

\paragraph{Learning Low-Dimensional Target Function.}
Low-dimensional target function classes, such as sparse parities~\citep{barak2022hidden,suzuki2023feature,glasgow2024sgd}, signal-noise models~\citep{allen2020towards,cao2022benign}, are widely adopted as theoretical benchmarks for studying a neural network's ability to perform feature learning. 
This work specifically focuses on the single-index model. A line of theoretical work has analyzed the learning of these models and has established key results on sample complexity. The required sample complexity is governed by either the \emph{information exponent} (for algorithms utilizing correlational information \citep{arous2021online,bietti2022learning,damian2023smoothing,mousavi2023gradient}) or the \emph{generative exponent} (for algorithms that employ suitable label transformations \citep{damian2024computational,lee2024neural,arnaboldi2024repetita,joshi2024complexity}). We discuss these sample complexity results in more detail in Section~\ref{sec:prelim}.

Our work is most closely related to \citet{oko2024pretrained,nishikawa2025nonlinear}, which lie at the intersection of ICL and the single-index model. Specifically, \citet{oko2024pretrained} show that a pretrained linear Transformer can effectively learn a single-index model in-context. More recent work by \citet{nishikawa2025nonlinear} establish an even smaller sample complexity and reveal that the nonlinear Transformer can perform test-time feature learning. A detailed comparison is provided in Section~\ref{sec:main}.
\section{Problem Setting}\label{sec:setting}
In this section, we provide a formal description of the key components we focus on: the ICL data distribution, the Mamba model, and the gradient-based pretraining algorithm.

\paragraph{Notation.}
We denote the $i$-th coordinate of a vector $\vv$ as $\vv[i]$, and the $(i,j)$-th coordinate of a matrix $\mM$ as $\mM [i,j]$. The matrix $\mathrm{diag}(\vv)$ represents a diagonal matrix with a vector $\vv$ on its main diagonal. We use $\odot$ for the element-wise product.
For any $k \in \mathbb{N}$, we denote the vectors with all entries equal to one and zero as $\vone_k$ and $\vzero_k$, respectively. We omit the subscript $k$ when it is clear.
\subsection{Data Distribution for In-Context Learning}
ICL aims to solve the \emph{task} of predicting the \emph{label} $y$ of a \emph{query} $\vx$ using a sequence of input-label pairs $\{(\vx_i,y_i)\}_{i \in [N]}$, which are referred to as \emph{context examples}. The model then utilizes a prompt, which is a sequence $(\vx_1, y_1, \dots, \vx_N, y_N, \vx)$ consisting of the context examples and the query, as its input. We focus on prompts constructed from the Gaussian single-index model, defined as follows.

\begin{definition} [Gaussian Single-Index Model]
Given a feature vector $\vbeta \in \R^d$, we draw input-label pairs $(\vx, y) \sim \gD_{\vbeta}$ as 
\begin{equation*}
    \vx \sim \gN(\vzero, \mI_d), \quad y = g_* (\langle \vbeta, \vx \rangle) +  \zeta, \quad \zeta \sim \mathrm{Unif}(\{-\tau, \tau\}),
\end{equation*}
where $g_*$ is a polynomial \emph{link function} and $\tau>0$ represents the noise level. For simplicity, we assume that $\mathbb{E}_{\rvz \sim \gN(0,1)}[g_*(\rvz)] = 0, \mathbb{E}_{\rvz \sim \gN(0,1)}[(g_*(\rvz))^2]=1$ and $\tau$ is a small enough constant.
\end{definition}
For each task, a prompt is constructed with a random choice of feature vectors.
\begin{definition}[ICL Data Distribution]\label{def:icl_data_distribution}
    For given a context length $N$, we define a data distribution $\gD(N)$ such that $(\vbeta, \{(\vx_i,y_i)\}_{i \in [N]}, \vx, y) \sim \gD(N)$ is constructed as follows.
    \begin{enumerate}[leftmargin=*]
        \item  We draw the feature vector $\vbeta \in \R^d$ uniformly from the unit sphere $S_r$ of a low-dimensional \emph{intrinsic feature space} with dimension $r$, defined as:
        \begin{equation*}
             S_r:= \left\{ \bm{\theta} \in \R^d : \norm{\bm{\theta}} = 1, \bm{\theta}[j] = 0 \text{ for all } j \notin \gI \right\},
        \end{equation*}
        for some unknown feature index set $\gI$ with $\abs{\gI} = r$.
        \item We sample $N$ context examples $\{(\vx_i,y_i)\}_{i \in [N]}$ and a query-label pair $(\vx, y)$ from $\gD_\vbeta$.
    \end{enumerate}
\end{definition}

Our task distribution exhibits a low-dimensional structure in two key aspects: (1) the label depends solely on the projection of the input onto the feature vector, and (2) feature vectors are supported on an $r$-dimensional subspace. We note that to achieve low prediction errors, it is crucial to extract both of these structures and estimate the link function $g_*$.

\subsection{Prediction Model Architecture}
{We consider a simplified prediction model for ICL consisting of three components: an input embedding, a single Mamba layer, and a multi-layer perceptron (MLP).}

\paragraph{Input Embedding.}
Given a prompt $(\vx_1, y_1, \dots, \vx_N, y_N, \vx )$ with context length $N$ and label $y$, we construct an input embedding $\mZ \in \R^{( \tilde d+1) \times (N+1)}$ as
\begin{equation*}
    \mZ = \begin{bmatrix}
    \phi(\vx_1) & \phi(\vx_2) &\dots & \phi(\vx_N) & \phi(\vx)\\
        y_1 &y_2 &\dots &y_N &0
    \end{bmatrix} = [\vz_1, \dots \vz_N ,\vz_{N+1}] \in \R^{(\tilde d +1) \times (N+1)},
\end{equation*}
where $\tilde d = 1+d+ \frac{d(d+1)}{2}$ and $\phi : \R^d \rightarrow \R^{\tilde d}$ is defined as 
\begin{equation*}
    \phi(\bm{\theta}) = \left[ 1, \bm{\theta}, \frac{\bm{\theta}\odot \bm{\theta}-\vone_d}{\sqrt{2}}, \bm{\theta}[1]\bm{\theta}[2], \dots, \bm{\theta}[d-1] \bm{\theta}[d] \right].
\end{equation*}
An input embedding similar to ours was also considered in the recent work by \citet{sun2025role}, who studied the in-context learning of high-order polynomial target functions. They showed this embedding can be implemented with a simple version of Gated Linear Unit (GLU) and demonstrated its efficacy for enabling linear Transformers to learn these functions in-context. Unlike \citet{sun2025role}, who repeatedly stacked a linear Transformer and a GLU layer, in our work, a single GLU-based embedding is sufficient due to the nonlinearity in Mamba and MLP layers. We discuss the efficacy of this input embedding in more detail in Section~\ref{sec:proof_overview_mamba}.

\begin{remark}
    Our input embedding is based on a basis for degree-2 polynomials in $\R^d$. Specifically, we use the standard basis of $\R^d$ for the construction of both the input embedding and the intrinsic feature space $S_r$. While extending our results to a general choice of $S_r$ with an arbitrary basis may require additional techniques, our setting remains valuable for studying Mamba's ability to learn low-dimensional structure. {For non-even link functions, a similar approach may also allow an extension to the standard input embedding $\phi(\vx)=\vx$ used in prior works~\citep{von2023transformers}.}
\end{remark}

\paragraph{One-Layer Mamba.}
Given an input embedding $\mZ = (\vz_1, \dots, \vz_{N+1}) \in \R^{(\tilde d +1) \times (N+1)}$, a one-layer Mamba model $\Mamba(\cdot; \bm\Theta)$ with parameters $\bm{\Theta}$ has sequential outputs $\vo_1, \dots, \vo_{N+1} \in \R^{\tilde d + 1}$ and hidden states for $i$-th channel $\vh_1^{(i)}, \dots, \vh_{N+1}^{(i)} \in \R^{\tilde d +1}$ defined as below:
\begin{align*}
    &\vh_l^{(i)} = \overline \mA_l \vh_{l-1}^{(i)}  + \overline \mB_l \vz_l[i] \in \R^{\tilde d+1} , \quad   \vo_l[i] = \mC_l^\top \vh_l^{(i)} \in \R,\\
    &\overline \mA_l = \exp \left( \Delta_l \mA \right) \in \R^{(\tilde d+1) \times (\tilde d+1)}, \quad \overline \mB_l = \left(\Delta_l \mA \right)^{-1} (\exp \left( \Delta_l \mA \right) -\mI_{\tilde d+1}) \Delta_l \mB_l \in \R^{\tilde d+1},
\end{align*}
where $\vh_0^{(i)} = \vzero_{\tilde d+1}$. Here, the components of the selection algorithm $ \mA, \mB_l, \mC_l, \Delta_l$ are chosen as 
\begin{equation*}
     \mA = -\mI_{\tilde d +1}, \quad \mB_l = \mW_B \vz_l, \quad  \mC_l = \mW_C \vz_l, \quad \Delta_l = \mathrm{softplus}\left( \vw^\top \vz_l + b \right),
\end{equation*}
with $\mW_B,  \mW_C \in \R^{(\tilde d+1) \times (\tilde d + 1)}, \vw \in \R^{\tilde d + 1}, b\in \R$. 
Then, the $l$-th output can be expressed as
\begin{equation}\label{eq:mamba_output}
    \vo_l = \sum_{j=1}^l G_{j,l}(\mZ) \vz_j \vz_j^\top \mW_B^\top \mW_C \vz_l, 
\end{equation}
where $G_{j,l}(\mZ) = \sigmoid \left( \vw^\top \vz_j + b \right) \prod_{k = j+1}^l \left( 1- \sigmoid \left( \vw^\top \vz_k + b \right)\right)$ with sigmoid function $\sigmoid(\cdot)$. It implies that Mamba involves two key mechanisms: \emph{nonlinear gating} $G_{j,l}(\mZ)$ and \emph{linear attention} with projection matrices $\mW_B$ and $\mW_C$. \citet{yang2024gated} refer the combination of these mechanisms as \emph{gated linear attention} and recent recurrent models including Mamba, mLSTM~\citep{beck2024xlstm}, and RWKV-6~\citep{peng2024eagle} can be viewed within this framework.

To ensure a tractable optimization guarantee, we further introduce the following simplifications to our model, while preserving the essential mechanism:
\begin{equation*}
    \mW_B^\top \mW_C 
    = \mathrm{diag}(\vgamma, 0), \quad \vw = \begin{bmatrix}
        \vzero_{\tilde d}\\
        \rho^{-1}
    \end{bmatrix}, 
\end{equation*}
where $\vgamma \in \R^{\tilde d}$ is a learnable parameter, and $\rho, b \in \R$ are fixed (as detailed in Theorem~\ref{thm:main}).
Our approach of merging the product of two learnable matrices into a single matrix and using sparse learnable parameters is a technique also adopted in the theoretical literature on optimization of attention mechanisms \citep{ahn2023transformers,zhang2024trained,mahankali2024one,kim2024nonlinear}. We fix the essential last coordinate of $\vw$ to a sufficiently small constant, closely approximating a random initialization, and set the remaining coordinates to zero to avoid unnecessarily complicated calculations. %
{
Additionally, we fix $b$ to a large negative value so that the product factor in $G_{j,l}$ is close to one. The remaining factor $\sigmoid(y_j/\rho+b)$ still transforms the context labels nonlinearly.
} Under this simplification, the last coordinate of the final output which serves as the input to the MLP can be expressed as follows:
\begin{equation*}
    \Mamba(\mZ ; \vgamma)[\tilde d+1, N+1] = \sum_{j=1}^N G_{j,N+1}(\mZ) y_j \phi(\vx_{j})^\top \left( \vgamma \odot \phi(\vx)\right).
\end{equation*}

\paragraph{Multi-Layer Perceptron.}
We use a two-layer MLP with ReLU activation, width $m$ and parameters $\vu, \vv, \va \in \R^m$  defined as follows:
\begin{equation*}
    \MLP(z ; \vu, \vv, \va) := \sum_{k=1}^m \vu[k] \relu \left( \vv[k] z + \va[k] \right).
\end{equation*}
We apply this MLP to the output of the Mamba layer, after normalizing it by its context length $N$. Then, the final output is given by
\begin{equation*}
     f(\mZ; \vgamma, \vu, \vv, \va) := \MLP\left( N^{-1} \Mamba (\mZ; \vgamma)[\tilde d+1,N+1]; \vu, \vv, \va \right)
\end{equation*}

\begin{remark}
   A similar structure to our models, which combines a sequence model with a MLP, has also been utilized in two closely related prior works. \citet{nishikawa2025nonlinear} follow a similar structure but use a softmax Transformer in place of Mamba. In contrast, \citet{oko2024pretrained} apply the MLP to the input embedding before a linear Transformer, rather than after the sequence model.
\end{remark}

Our goal for ICL is to find parameters $\vgamma\in \R^{\tilde d}, \vu, \vv, \va \in \R^m$, and context length $N$, achieving a small ICL test error, which is defined as 
\begin{equation*}
    \gR_N(\vgamma, \vu, \vv, \va) := \mathbb{E}_{(\mZ,y)\sim \gD(N)}[\abs{f(\mZ, \vgamma, \vu, \vv, \va) - y}].
\end{equation*}
Here, we abuse notation and use $(\mZ,y)$ to denote the input embedding and label for a prompt sampled from the ICL data distribution $\gD(N)$.
More precisely, we are interested in the sample complexity of context examples required for the parameters learned from pretraining to achieve a low ICL test error.

\subsection{Pretraining Algorithm}
{
Our prediction model is pretrained on a set of $T_{\mathrm{pt}}=T_1+T_2$ tasks with context length $N_\pt$ drawn from $\gD(N_\pt)$. For each task $t\in[T_{\mathrm{pt}}]$, we have an input embedding $\mZ^t$ constructed from context examples $\{(\vx_i^t,y_i^t)\}_{i\in[N_\pt]}$ and a query-label pair $(\vx^t,y^t)$ with a feature vector $\vbeta^t$. Then, our training losses can be written as
}
\begin{equation*}
    L_l (\vgamma, \vu, \vv, \va) := \frac{1}{T_l}\sum_{t = T_{l-1}+1}^{T_{l-1} + T_l} \left(f\left(\mZ^t ; \vgamma, \vu, \vv, \va \right) - y^t\right)^2,
\end{equation*}
for $l=1,2$ and $T_0 = 0$. Using these objectives, we employ Algorithm~\ref{alg:pretraining} which proceeds in two stages: Stage I recovers intrinsic feature structures, enabling Stage II to perform in-context link estimation conditioned on these features.
This approach enables tractable analysis and has been widely adopted in prior works \citep{ba2022high,damian2022neural}.
\begin{enumerate}[leftmargin = *]
\item In Stage I, we only train the Mamba layer parameter $\vgamma$, starting from proper initialization. Our training objective is $\ell_2$-regularized loss $L_1(\vgamma, \vu, \vv, \va) + \frac{\lambda_1}{2} \norm{\vgamma}^2$. Due to the non-linearity introduced by the MLP, this objective is non-convex. To make the training dynamics tractable, we apply one-step gradient descent, following the approaches studied in the literature of feature learning \citep{ba2022high,damian2022neural}. As we describe in Section~\ref{sec:proof_overview_mamba}, a single step update is sufficient to capture the low-dimensional structure of the feature vectors.
\item In Stage II, we fix the Mamba layer parameter $\vgamma^*$ obtained from Stage I and optimize the outer layer $\vu$ of MLP on $\ell_2$-regularized loss $L_2(\vgamma^*, \vu, \vv^*, \va^*) + \frac{\lambda_2}{2} \norm{\vu}^2$ with reinitialized inner layer parameters $\vv^*, \va^*$. This induces a convex problem that gradient-based methods can solve. As we show in Section~\ref{sec:proof_overview_MLP}, the optimized MLP is capable of estimating the link function $g_*$.
\end{enumerate}
\begin{algorithm}[h]
\SetKwInOut{Input}{Input}
\SetKwInOut{Output}{Output}
\SetKwBlock{StageOne}{Stage I: Gradient descent on Mamba layer}{}
\SetKwBlock{StageTwo}{Stage II: Optimization of MLP Layer}{}
\SetKw{Initialize}{Initialize}

\Input{Learning rate $\eta$, weight decay $\lambda_1,\lambda_2$, context length $N_{\pt}$, the number of tasks $T_1,T_2$, initialization scale $\gamma, \rho, b$.}

\StageOne{
\Initialize{$\vgamma^{(0)} = (\gamma^2, 1, \dots, 1, \gamma, \cdots, \gamma), \vu^{(0)} = m^{-1} \vone_m, \vv^{(0)} = \vone_m, \va^{(0)} = \vzero_m$.\label{alg:line:symmetric}}\\
$\vgamma^* \leftarrow \vgamma^{(0)}-\eta\nabla_{\vgamma}\left(L_1(\vgamma^{(0)},\vu^{(0)},\vv^{(0)},\va^{(0)})+ \frac{\lambda_1}{2} \norm{\vgamma}^2 \right)$. \label{alg:line:onestepgradient}
}	

\StageTwo{
\Initialize{$\vv^* \sim \mathrm{Unif}\left(\{\pm1\}^m\right)$, $\va^* \sim \mathrm{Unif}([-1,1]^m)$. }\\
$\vu^* \leftarrow\argmin_\vu \left( L_2(\vgamma^*, \vu, \vv^*, \va^*)+\frac{\lambda_2}{2} \norm{\vu}^2\right)$. \label{alg:line:ridge}
}

\Output{Prediction function $f( \cdot ;\vgamma^*, \vu^*,\vv^*,\va^*)$.}
\caption{Gradient-based Pretraining of the Mamba Model} 
\label{alg:pretraining}
\end{algorithm}

\section{Mamba Efficiently Learns Single-Index Models In-Context}
In this section, we present our theoretical results. Our analysis focuses on the asymptotic dependencies on the input dimension $d$, while the link function $g_*$ is fixed. {For our analysis, we assume that the intrinsic dimension satisfies $r\geq(\log d)^{C_r}$ for a sufficiently large constant $C_r>0$. This technical condition controls residual terms in our analysis, and removing it appears to require additional techniques. This condition still allows the low-dimensional regime $r\ll d$.} {We let} $N^*$ and $T^*$ be the maximum admissible context length and the number of pretraining tasks, respectively. We assume that $N^* , T^* \leq d^{C^*}$ for some large constant $C^*>0$. We use the standard asymptotic notation $\bigO(\cdot), \Omega(\cdot), \Theta(\cdot), o(\cdot)$ to express dependencies on $d$, and $\bigOtilde(\cdot), \tilde \Omega(\cdot), \tilde \Theta (\cdot)$ to hide logarithmic factors of $d$.

\subsection{Preliminaries}\label{sec:prelim}

We first provide backgrounds on learning Gaussian single-index models, which are essential for understanding our main result.
Let $\He_i(z) = (-1)^i e^{\frac{z^2}{2}}\frac{\mathrm{d}^i}{\mathrm{d} z^i}e^{-\frac {z^2}{2}}$ denote the probabilist's Hermite polynomials. Then, the set $\{\He_i(z)/\sqrt{i!}\}_{i \in \mathbb{N} \cup \{0\}}$ forms an orthonormal basis of the $L^2$ space with respect to the Gaussian measure and serves as a key technical tool for the analysis of Gaussian single-index models. We now introduce two key terms relevant to the sample complexity of learning.

\begin{definition}
    For any function $h:\R \rightarrow \R$ which is $L^2$-integrable with respect to the Gaussian measure, we express its Hermite expansion as
    \begin{equation*}
        h(z) = \sum_{i=0}^\infty \frac{H(h,i)}{i!} \He_i(z), \quad H(h,i):= \mathbb{E}_{\rvz \sim \gN(0,1)}[h(\rvz) \He_i(\rvz)].
    \end{equation*}
    We also define the following quantities:
    \begin{itemize}[leftmargin=*]
        \item We define $\deg(h)$ as the degree of $h$, if it is a polynomial.
        \item The \emph{information exponent}~\citep{arous2021online,damian2023smoothing} of $h$ is defined as 
        \begin{equation*}
            \ie(h):= \min \{ i \in \mathbb{N} : H(h,i) \neq 0\}.
        \end{equation*}
        It implies that $\mathbb{E}_{\rvz \sim \gN(0,1)}[h(\rvz) \He_k(\rvz)] = 0$ for any $k \in \mathbb{N}$ with $k< \ie(h)$.
        \item The \emph{generative exponent}~\citep{damian2024computational} of $h$ is defined as the lowest possible information exponent after an $L^2$ transformation. It is formally defined as:
        \begin{equation*}
            \ge(h):= \min_{\gT\in L^2(\mathbb{P}_h)} \min\{i \in \mathbb{N} : H( \gT \circ h,i) \neq 0\},
        \end{equation*}
{
        where $L^2(\mathbb{P}_h)$ is the set of $L^2$-integrable functions with respect to $\mathbb{P}_h$. Here, $\mathbb{P}_h$ is the push-forward measure of the Gaussian measure by $h$.
}
    \end{itemize}
\end{definition}

While the definition of the generative exponent may seem difficult to apply at first glance, \citet{lee2024neural} provides a characterization of the generative exponent for polynomials.
\begin{lemma}[Proposition 6 in \citet{lee2024neural}]\label{lemma:ge}
    For a polynomial link function $g_*$, the generative exponent is characterized as $\ge(g_*)=2$ if $g_*$ is an even function, and $\ge(g_*)=1$ otherwise.
\end{lemma}

From the definition, $\ge(g_*)\leq \ie(g_*)\leq \deg(g_*)$ and Lemma~\ref{lemma:ge} implies that the gap between these three terms can be arbitrarily large depending on the choice of $g_*$\footnote{For example, consider $g_*(z) = \He_q(z) + \He_p(z)$ with $1\ll q \ll p$.}. With a slight abuse of notation, we use $\Theta(\deg(g_*))$ to denote a quantity that is bounded by a universal constant multiple of $\deg(g_*)$. We also use $\Theta(\ie(g_*))$ and $\Theta(\ge(g_*))$, in similar manners.

\paragraph{Sample Complexity of Learning Single-Index Models.} Previous works have established the sample complexity of various methods for learning a Gaussian single-index model. For example, kernel methods, which lack an adaptive basis, require at least $d^{\deg (g_*)}$ samples~\citep{ghorbani2021linearized,donhauser2021rotational}. In contrast, adaptive methods such as gradient-based methods on two-layer neural networks can achieve a sample complexity of  $\bigOtilde \left(d^{\Theta(\ie(g_*))}\right)$ by learning an adaptive feature map~\citep{arous2021online,ba2022high,damian2022neural,damian2023smoothing,dandi2023two}. These approaches fall under the category of CSQ algorithms, and in this category, a sample complexity that depends on the information exponent is inevitable \citep{damian2022neural}. However, recent works show that a nonlinear transformation introduced by data reuse \citep{arnaboldi2024repetita,lee2024neural} enables the algorithm to move into the broader class of Statistical Query (SQ) algorithms. This transformation allows the ``effective'' information exponent to be lowered to the generative exponent, thereby achieving a sample complexity of $\bigOtilde\left(d^{\Theta(\ge(g_*))}\right)$.

\subsection{Main Result}\label{sec:main}
We now present our main result, which provides a theoretical characterization of the pretraining and test-time sample complexities for achieving low ICL errors.

\begin{theorem}\label{thm:main}
Let $f(\cdot ; \vgamma^*, \vu^*, \vv^*, \va^*)$ be the Mamba model pretrained using Algorithm~\ref{alg:pretraining}. We assume the following conditions hold for its hyperparameters:
\begin{itemize}[leftmargin=*]
    \item The context length is $N_\pt = \tilde\Omega\left(r^{3 \ge (g_*)} d^8 \vee T_1^2 d^4 \right)$.
    \item The number of pretraining tasks are $T_1 = \tilde\Omega\left(r^{3\ge(g_*)}d^6\right)$ and $T_2 = \tilde\Omega\left(r^{3 \ge (g_*)}\right)$.
    \item The MLP width is $m = \tilde\Omega \left( r^{4 \ge(g_*)}\right)$.
    \item {The fixed weights are $\rho = \Theta\left((\log d)^{C_\rho}\right)$ and $b = -C_b \log d$, and the initialization scale is $\gamma = \Theta((\log d)^{-C_\gamma})$ for sufficiently large constants $C_\gamma, C_\rho, C_b > 0$.}
\end{itemize}
Then, there exist hyperparameters $\lambda_1, \lambda_2,$ and $\eta$ such that with probability at least 0.99 over the training data and random initialization, the following holds:
If the test prompt length satisfies $N_{\mathrm{test}} = \tilde\Omega\left(r^{3\ge(g_*)}\right)$, then the test error $\gR_{N_{\mathrm{test}}}(\vgamma^*, \vu^*, \vv^*, \va^*)$ is bounded by $\tau + o(1)$.
\end{theorem}

We discuss our sample complexity results in comparison with other methods, including regression on test prompts and prior theoretical works \citep{oko2024pretrained, nishikawa2025nonlinear}. We summarize these results in Table~\ref{table} and highlight the following key points:
\paragraph{Adaptation to Low-Dimensional Structure.} 
{
Our test-time sample complexity depends on the intrinsic dimension of the feature vectors $r$, rather than the ambient dimension $d$.
} This is consistent with prior works \citep{oko2024pretrained, nishikawa2025nonlinear} that also demonstrate a dependence on intrinsic dimensionality. In contrast, the sample complexities of various regression algorithms we have discussed depend on the full input dimension $d$. This difference arises because pretrained models can learn the low-dimensional structure of the intrinsic feature space during pretraining.
\paragraph{Test-Time Feature Learning.} The dependence of the sample complexity on the intrinsic dimension $r$ in the work of \citet{oko2024pretrained} is controlled by the degree of the link function $g_*$. This means that while their approach is more efficient than simple regression on full dimensions, its performance remains close to that of kernel methods on intrinsic dimensions. In contrast, our result depends on the generative exponent $\ge(g_*)$, instead of its degree. %
{
This implies that Mamba's efficient in-context learning is enabled not just by its ability to learn an intrinsic feature space, but also by a process called \emph{test-time feature learning}, which allows the model to extract features directly from the context.
} The same process also works for the softmax Transformers considered in \citet{nishikawa2025nonlinear} and thus achieved a similar sample complexity. However, they perform test-time feature learning through different mechanisms: Mamba relies on nonlinear gating, while Transformer uses softmax attention.

\noindent
\begin{minipage}[t]{0.45\textwidth}
    \paragraph{Improved Pretraining Sample Complexity.} The conditions for the pretraining in our theorem can be satisfied with a pretraining sample complexity of $N_\pt = \Theta\left(d^{\Theta(\ge(g_*))}\right)$. In contrast, the pretraining sample complexities in previous works \citep{oko2024pretrained, nishikawa2025nonlinear} are governed by the information exponent, which can lead to a suboptimal rate in the worst case. This improvement is due to the nonlinearity of the MLP, as we discuss in detail in Section~\ref{sec:proof_overview}. Note that since these sample complexity bounds are sufficient conditions, it remains open whether Mamba is 
\end{minipage}
\hfill
\begin{minipage}[t]{0.53\textwidth} %
    \centering
    \captionof{table}{Summary of sample complexity results for regression algorithms on test prompt and prior works on ICL \citep{oko2024pretrained, nishikawa2025nonlinear}.}\label{table}
    \resizebox{\linewidth}{!}{ 
    \begin{tabular}{c c c}
    \toprule
    \multicolumn{3}{c}{\textbf{Regression on Test Prompt}} \\
    \midrule
    Kernel & CSQ & SQ \\
    \midrule
    $d^{\Theta(\deg(g_*))}$ & $d^{\Theta(\mathrm{ie}(g_*))}$ & $d^{\Theta(\mathrm{ge}(g_*))}$ \\
    \midrule\midrule
    \multicolumn{3}{c}{\textbf{In-Context Learning}} \\
    \midrule
    Linear Transformer & Softmax Transformer & Mamba\\
    \midrule
    \citet{oko2024pretrained} & \citet{nishikawa2025nonlinear} & \textcolor{darkred}{This Work}\\
    \midrule
    \emph{Pretrain:} $d^{\Theta(\mathrm{ie}(g_*))}$ & 
    \emph{Pretrain:} $d^{\Theta(\mathrm{ie}(g_*))}$ & 
    \emph{Pretrain:} \textcolor{darkred}{$d^{\Theta(\mathrm{ge}(g_*))}$} \\
    \emph{Test:} $r^{\Theta(\deg(g_*))}$ &
    \emph{Test:} $r^{\Theta(\mathrm{ge}(g_*))}$ &
    \emph{Test:} \textcolor{darkred}{$r^{\Theta(\mathrm{ge}(g_*))}$} \\
    \bottomrule
    \end{tabular}
    }
\end{minipage}
theoretically more efficient than Transformers in pretraining. While we leave this theoretical investigation as a future direction, our experiments in Section~\ref{sec:exp} suggest this is indeed the case.

\section{Proof Overview}\label{sec:proof_overview}
In this section, we provide an overview of the proof for our theorem. The proof consists of three main parts: an analysis of one-step gradient descent on the Mamba layer, the optimization of the MLP, and a test error analysis. The formal proofs for each of these steps are provided in Appendices~\ref{appendix:mamba}, \ref{appendix:mlp}, and \ref{appendix:test_error}, respectively. In the following, we introduce the key ideas behind each step, providing a detailed discussion of our theoretical components.
\subsection{One-Step Gradient Descent on the Mamba Layer}\label{sec:proof_overview_mamba}
{
We show that the pretrained parameter $\vgamma^*$ recovers the intrinsic feature space by assigning significantly larger weights to embedding coordinates supported on the feature index set $\gI$ than to those involving irrelevant coordinates. This identifies the structure shared across tasks; the context examples then provide the information about the particular task's feature vector $\vbeta$. The resulting test-time feature learning is summarized in the following proposition (formally stated in Proposition~\ref{proposition:updated_mamba}):
}
\begin{proposition}[Informal]\label{proposition:test-time_feature_learning}
    For a sampled ICL input embedding $\mZ$ with context length $N = \tilde \Omega\left(r^{3 \ge(g_*)}\right)$, query $\vx$, and feature vector $\vbeta$, the following holds with high probability:
    \begin{equation*}%
        N^{-1}\Mamba(\mZ;\vgamma^*) \approx P_1 + P_2 \left( \frac{\langle \vbeta, \vx \rangle}{r}\right)^{\ge(g_*)},
    \end{equation*}
    where $P_1$ and $P_2$ are independent of the data.
\end{proposition}

Assuming a negative bias $b$ with sufficiently large absolute value, and a large enough number of tasks $T_1$ and context length $N_\pt$, the updated parameter $\vgamma^*$ can be approximated as follows:
{
\begin{align*}
    \vgamma^* &\approx 2\eta \mathbb{E}_{(\mZ,y)\sim \gD(N_\pt)}[y \nabla_\vgamma f(\mZ ; \vgamma^{(0)}, \vu^{(0)},\vv^{(0)},\va^{(0)})]\\
    &\approx 2\eta \mathbb{E}_{\substack{\vbeta \sim \mathrm{Unif}(S_r)\\(\vx,y) \sim \gD_{\vbeta}}}
    \left[y\,\mathbbm{1}\!\left[\left\langle\vc_\vbeta,\vgamma^{(0)}\odot\phi(\vx)\right\rangle>0\right]
    \vc_\vbeta\odot\phi(\vx)\right],
\end{align*}
}
where $\vc_\vbeta := \mathbb{E}_{(\vx, y) \sim \gD_\vbeta} \left[ y \sigma(y/\rho + b) \phi(\vx) \right]$ corresponds to a simplified expectation over context examples, neglecting the effect of ``forgetting'' in the gating mechanism.
{
The negative bias makes cumulative forgetting negligible over the admissible context lengths, while concentration over context examples and pretraining tasks justifies the population approximation. Appendix~\ref{appendix:mamba} bounds these errors together with the contribution of the initial model output.
}
\paragraph{The Role of Gating and Input Embedding.} For the proof of Proposition~\ref{proposition:test-time_feature_learning} and our test-time sample complexity, nonlinear transformation introduced by gating mechanism plays a crucial role. In the absence of a gating mechanism and with only a linear attention, the term $\vc_\vbeta$ is replaced by $\mathbb{E}_{(\vx,y) \sim \gD_\vbeta}[y\phi(\vx)]$ and this term vanishes when $\ie(g_*)>2$. This is a consequence of our input embedding using Hermite polynomials only up to the second degree, in combination with Stein's lemma. As a result, pretraining in this case is unable to learn useful information.
However, we prove that the gating mechanism ensures $c_\vbeta$ is non-zero, thereby enabling the model to extract information. This is because the nonlinear transformation introduced by the gating mechanism reduces the information exponent to the generative exponent: $\ie(g_* \sigma(g_*/\rho +b)) = \ge(g_*)$. This reduction, combined with Lemma~\ref{lemma:ge} and our input embedding, makes information extraction possible. 
{A similar approach may also allow an extension to the standard input embedding $\phi(\vx)=\vx$ for non-even link functions.} Reducing the information exponent to the generative exponent crucially affects the achievement of a test-time sample complexity below the CSQ lower bounds. \citet{nishikawa2025nonlinear} shows that the softmax operator in Transformer can also perform a nonlinear transformation on the label, which reduces the information exponent. This highlights a key difference in the mechanisms used by Mamba and Transformer to achieve this result.%
\paragraph{Improved Pretraining Sample Complexity.}
{
{Our pretraining sample complexity scales as $d^{\Theta(\ge(g_*))}$. This improvement relies on the nonlinearity of the MLP, whose ReLU activation introduces the indicator in the gradient.}
}
If the indicator $\mathbbm{1}[\cdot]$ inside the expectation is 1 with high probability, then the updated parameter $\vgamma^*$ can be approximated as $\vgamma^* \approx 2\eta \mathbb{E}_{\vbeta \sim \mathrm{Unif}(S_r)}\left[\mathbb{E}_{(\vx,y \sim \gD_\vbeta)}[y \vc_\vbeta \odot \phi(\vx)]\right]$ and this close to zero when $\ie(g_*)>2$, leading to less information gain.
However, we show that the indicator deviates significantly from a constant value. We formally prove that this deviation allows the indicator to reduce the information exponent when multiplied by the label $y$, thereby inducing a pretraining sample complexity not governed by $\ie(g_*)$. While \citet{nishikawa2025nonlinear} employ an architecture with a similar structure to ours—an MLP following a softmax Transformer—they do not achieve the same improvement. This is because their use of softmax places the model in a regime where a key indicator function is 1 with high probability, which is sufficient in their case as softmax generates the necessary higher-order functions of the input, while our analysis only uses up to second-order functions. In addition, our observation for this improvement cannot be directly applied to the work of \citet{oko2024pretrained} due to a key architectural difference: applying the MLP layer in the input embedding rather than at the output layer.

\subsection{Optimization of the MLP and Test Error Analysis}\label{sec:proof_overview_MLP}
In our analysis of Stage II pretraining, we show that the MLP can fit the link function $g_*$. We first construct an outer layer parameter $\vu'$ such that the loss $L_2(\vgamma^*, \vu', \vv, \va)$ is sufficiently small and the norm $\norm{\vu'}$ is well-bounded. Our construction is based on the techniques in \citet{damian2022neural}, which constructed a ReLU network approximating monomials. This allows our model to learn high-order polynomials with a few layers, in contrast to the multi-layer approach in \citet{sun2025role}. 
{This is possible because Lemma~\ref{lemma:ge} implies that $g_*(z)$ is a polynomial of $z^{\ge(g_*)}$, allowing us to construct an MLP that approximates $g_*(\langle\vbeta,\vx\rangle)$ when the input is provided in the form given in Proposition~\ref{proposition:test-time_feature_learning}.}

{
For a suitable $\lambda_2>0$, we show that the optimized output weights $\vu^*$ attain empirical absolute loss $\tau+o(1)$ with a well-bounded norm.
}
{
Next, we show that the trained model achieves a small test error with context length $N_\pt$ by applying a standard generalization bound based on Rademacher complexity, which is applicable due to a well-bounded norm $\norm{\vu^*}$. Lastly, we extend this error bound to a general context length $N_{\mathrm{test}}=\tilde\Omega\left(r^{3\ge(g_*)}\right)$. This is possible because Proposition~\ref{proposition:test-time_feature_learning} implies that prompts with $N_{\mathrm{test}}$ context examples and $N_\pt$ context examples give similar normalized Mamba outputs, given the same feature vector $\vbeta$ and query.
}

\section{Experiments}\label{sec:exp}
To support our theoretical findings, we conduct experiments beyond our simplified setting using standard architectures and optimization algorithms on NVIDIA RTX A6000 GPUs.

We pretrain and evaluate both Transformer and Mamba on our data distribution. Our base configuration uses a link function $g_*(z) = \He_3 (z) / \sqrt{6}$, an intrinsic dimension of $r=8$, and an ambient dimension $d=32$. We employ a 6-layer GPT-2 model \citep{radford2019language} with 8 attention heads and a 12-layer Mamba model. To ensure a fair comparison, both models have an embedding dimension of 256 and a similar number of parameters. The overall experimental settings for pretraining follow those of prior works \citep{garg2022can,park2024can},\footnote{We build upon the official codebases of these works (MIT and Apache 2.0 licenses).} including the optimization algorithm and a pretraining context length $N_\pt = 64$. We also conduct kernel ridge regression on the intrinsic feature space to serve as a baseline for understanding the effect of feature learning. For this, we use a Gaussian RBF kernel with a bandwidth of 1 and a ridge parameter of 1. For evaluation, we measure the prediction error using squared error, with the number of context examples ranging from 1 to 40. We estimate the test error on 1024 randomly sampled tasks, using 2048 independent prompts for each task, and represent the results with the mean and standard deviation over these tasks. 

To validate our results on the influence of problem parameters, we analyze trends by varying parameters from our base configuration: the ambient dimension to $d=16$, the intrinsic dimension to $r = 16$, and the pretraining context length to $N_\pt = 16$.
Figure~\ref{fig:d} demonstrates the influence of the ambient dimension $d$. Both Transformer and Mamba models exhibit comparable performance that is rarely affected by the ambient dimension $d$.
In contrast, when the intrinsic dimension $r$ is increased, both models exhibit performance degradation, while their performance remains comparable (Figure~\ref{fig:r}). This suggests that both models mainly utilize information from the intrinsic feature space. In addition, these methods outperform kernel methods, even when we restrict the input of the kernel method to the intrinsic feature space. This observation aligns with our finding that Mamba, similar to Transformers, not only benefits from its adaptation to the intrinsic feature space but also performs test-time feature learning. Lastly, when using a small pretraining context length $N_\pt = 16$, the Transformer's performance deteriorates significantly, while that of Mamba remains comparable to its performance with $N_\pt=64$ (Figure~\ref{fig:Npt}). This observation aligns with our pretraining sample complexity result, which is lower than that of the Transformer established by \citet{nishikawa2025nonlinear}.%
\begin{figure}[h]
    \centering
    \begin{subfigure}[t]{0.333\textwidth}
        \centering
        \includegraphics[width = \linewidth]{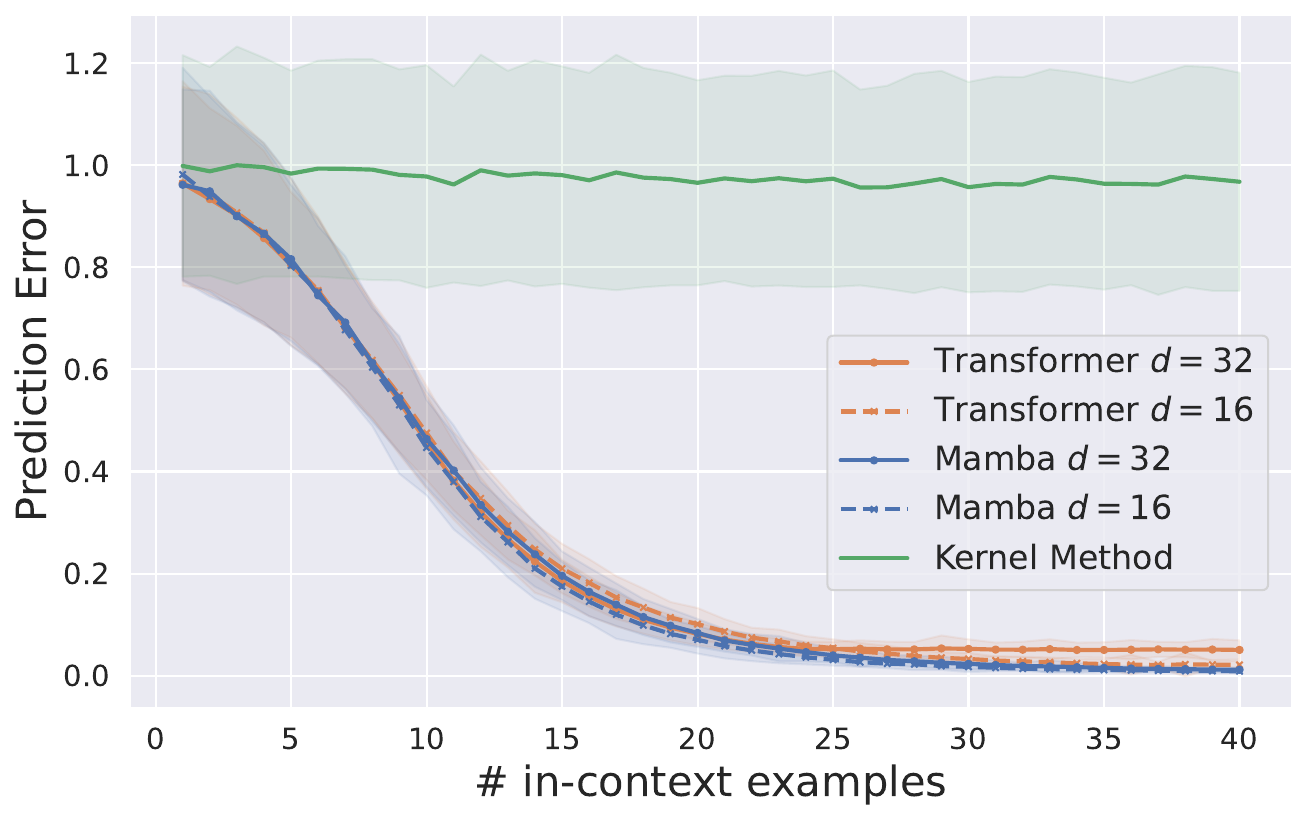}
        \caption{$d = 32$ vs. $d=16$}
        \label{fig:d}
    \end{subfigure}%
    \begin{subfigure}[t]{0.333\textwidth}
        \centering
        \includegraphics[width = \linewidth]{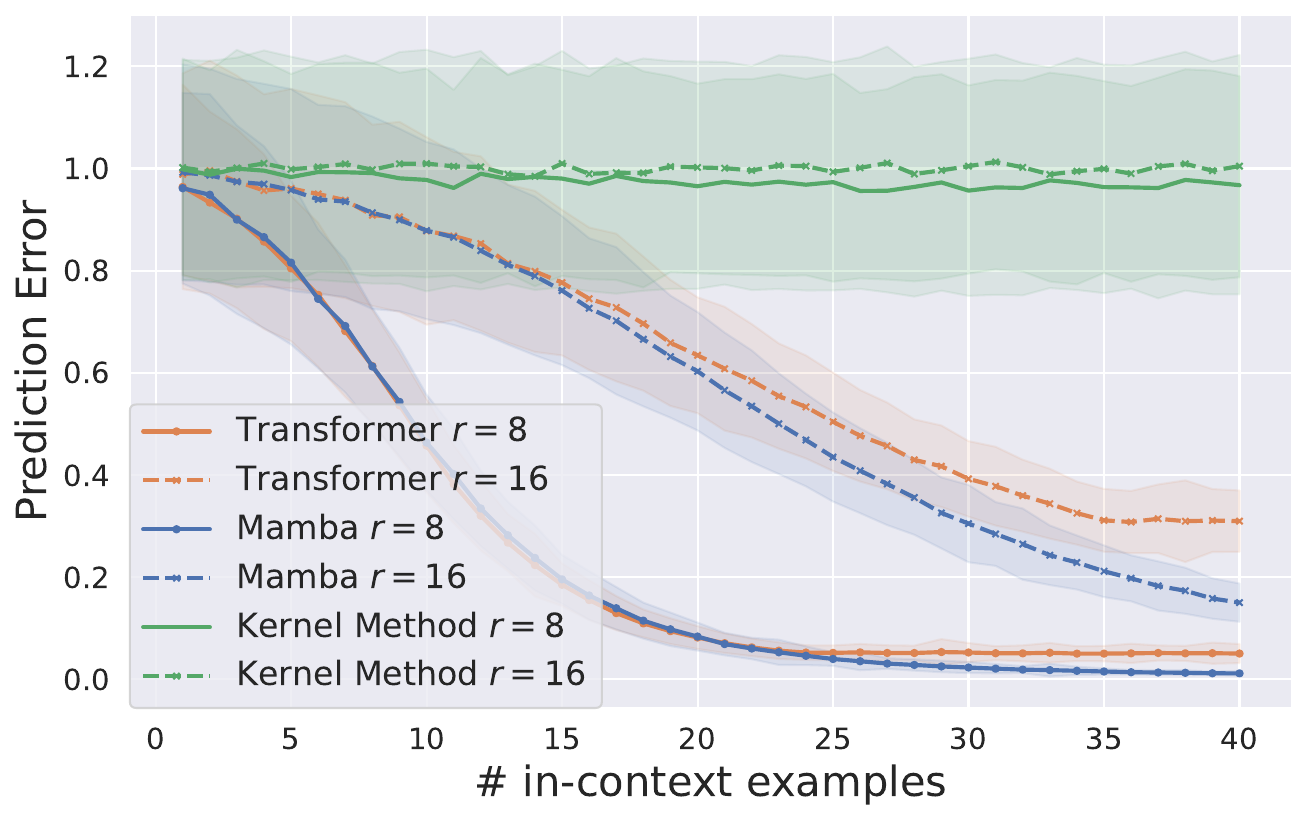}
        \caption{$r=8$ vs. $r=16$}
        \label{fig:r}
    \end{subfigure}%
    \begin{subfigure}[t]{0.333\textwidth}
        \centering
        \includegraphics[width = \linewidth]{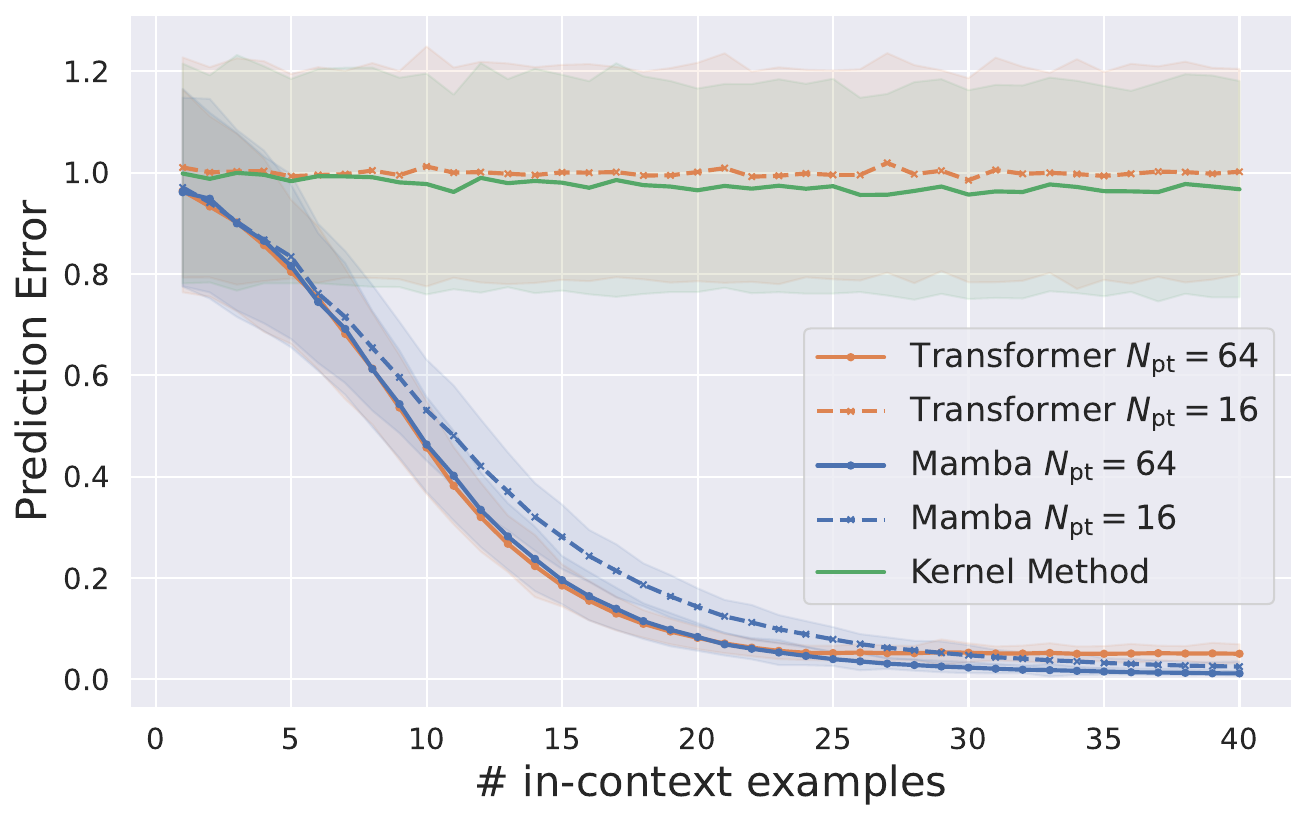}
        \caption{$N_\pt = 64$  vs. $N_\pt = 16$}
        \label{fig:Npt}
    \end{subfigure}
    \caption{Comparison of prediction error for in-context learning with Transformer and Mamba models, and kernel regression across different problem parameter settings.}
    \label{fig}%
\end{figure}%
\section{Conclusion}\label{sec:conclusion}
We investigated Mamba's in-context learning capability by focusing on a Gaussian single-index model. We proved that Mamba, when pretrained with gradient-based optimization, can efficiently learn in-context through a mechanism we termed test-time feature learning. Our derived test-time sample complexity is comparable to that of the softmax Transformer, a result established by \citet{nishikawa2025nonlinear} and also surpasses the CSQ lower bound. Our analysis reveals that Mamba's gating mechanism is a key factor in enabling feature learning and strong performance. We also presented experimental results to support our findings. 

{
We discuss the limitations of our work and future directions. First, our analysis uses a simplified data distribution, model architecture, and training procedure; extending our results beyond this setting is an important direction. Second, our experiments suggest that Mamba is more sample-efficient than Transformers in pretraining, but a theoretical explanation for this difference remains open. Third, our analysis assumes a negligible effect of forgetting, which can benefit tasks with outliers \citep{li2025can}. Understanding the interplay between forgetting and our feature-learning mechanism is an interesting direction.
}

\section*{Acknowledgement}
JO acknowledges partial support by the Institute for Information \& communications Technology Planning \& Evaluation (IITP) grant funded by the Korean government (MSIT) (No. RS-2019-II190075, Artificial Intelligence Graduate School Program(KAIST)).
WH was supported by JSPS KAKENHI (24K20848) and JST BOOST (JPMJBY24G6). TS was partially supported by JSPS KAKENHI (24K02905) and JST CREST (JPMJCR2115).
This work was partially supported by JST ERATO (JPMJER2601). 
This research is supported by the National Research Foundation, Singapore, Infocomm Media Development Authority under its Trust Tech Funding Initiative, and the Ministry of Digital Development and Information under the AI Visiting Professorship Programme (award number AIVP-2024-004). Any opinions, findings and conclusions or recommendations expressed in this material are those of the author(s) and do not reflect the views of National Research Foundation, Singapore, Infocomm Media Development Authority, and the Ministry of Digital Development and Information.

\bibliography{reference.bib}
\bibliographystyle{plainnat}

\clearpage
\appendix

\tableofcontents
\clearpage
\allowdisplaybreaks
\section{Proof Preliminaries}

\paragraph{Notation.} We introduce the following additional notation for ease of presentation. We use $\mathbbm{1}[\cdot]$ to represent indicator function. We write $a \lesssim b$, $a \gtrsim b$, and $a \asymp b$ to denote that $a = \bigO(b)$, $a = \Omega(b)$, and $a = \Theta(b)$, respectively. We also use $\poly(d)$ and $\polylog(d)$ to represent a sufficiently large polynomial in $d$ and $\log d$, respectively. Lastly, we use $o(1/\polylog(d))$ to represent a quantity that decreases faster than $(\log d)^{-C}$ for any constant $C>0$. Lastly, with a slight abuse of notation, we use the asymptotic notation we introduced to represent a vector when its norm satisfies the corresponding bound.

\subsection{Simplification of Mamba Output}\label{sec:mamba_simple}
The following lemma immediately implies \eqref{eq:mamba_output}.
\begin{lemma}\label{lemma:mamba_output}
    Given a prompt $(\vx_1, y_1, \dots, \vx_N, y_N, \vx )$ and its input embedding $\mZ \in \R^{(\tilde d+1) \times (N+1)}$, let $\Mamba(\mZ ; \bm{\Theta}) = (\vo_1, \dots, \vo_{N+1})\in \R^{(\tilde d+1)\times (N+1)}$ be outputs and $\vh_l^{(i)}$'s be hidden states. For each $i \in [\tilde d+1]$ and $l \in [N+1]$, we have
    \begin{equation*}
        \vh_l^{(i)} = \sum_{j=1}^l G_{j,l}(\mZ) \vz_j[i]  \mW_B \vz_j,\quad \vo_l[i] = \sum_{j=1}^l G_{j,l}(\mZ) \vz_j[i] \vz_j^\top \mW_B^\top \mW_C \vz_l,
    \end{equation*}
    where $G_{j,l}(\mZ) = \sigmoid \left( \vw^\top \vz_j + b \right) \prod_{k = j+1}^l \left( 1- \sigmoid \left( \vw^\top \vz_k + b \right)\right)$.
\end{lemma}
\begin{proof}[Proof of Lemma~\ref{lemma:mamba_output}]
    For each $l \in [N+1]$, we have
    \begin{equation*}
        \overline \mA_l = \exp \left(- \mathrm{softplus}(\vw^\top \vz_l+b)\mI_{\tilde d+1}\right) = \frac{1}{1 + \exp (\vw^\top \vz_l + b)} \mI_{\tilde d+1} = \left(1- \sigmoid \left(\vw^\top \vz_l + b \right) \right) \mI_{\tilde d+1}
    \end{equation*}
    and
    \begin{equation*}
        \overline \mB_l = -  \left( \overline \mA_l  - \mI_{\tilde d+1} \right) \mW_B \vz_l\\
        = \sigmoid\left (\vw^\top \vz_l + b \right ) \mW_B \vz_l.
    \end{equation*}
    We fix any $i \in [\tilde d+1]$ and we will prove by applying induction on $l\in [N+1]$.
    Let us first consider the case $l = 1$. We have
    \begin{align*}
        \vh_1^{(i)} &= \left(1 -  \sigmoid \left(\vw^\top \vz_1 + b \right)\right) \vh_0 ^{(i)} + \sigmoid \left(\vw^\top \vz_1 + b \right) \mW_B \vz_1 \vz_1[i] \\
        &=  \sigmoid \left(\vw^\top \vz_1 + b \right) \mW_B \vz_1 \vz_1[i]\\
        &= G_{1,1}(\mZ) \vz_1[i] \mW_B \vz_1
    \end{align*}
    and 
    \begin{equation*}
        \vo_1[i] = \left( \mW_C \vz_1 \right)^\top  \vh_1^{(i)} = \sigmoid \left(\vw^\top \vz_1 + b \right)  \vz_1[i] \vz_1^\top \mW_B^\top \mW_C \vz_1 = G_{1,1}(\mZ) \vz_1[i] \vz_1^\top \mW_B^\top \mW_C \vz_1. 
    \end{equation*}
    Therefore, desired conclusions hold for the case $l = 1$. 
    
    Next, we assume that our conclusion holds for $l < N+1$.
    \begin{align*}
        \vh_{l+1}^{(i)} &= \left(1 -  \sigmoid \left(\vw^\top \vz_{l+1} + b \right)\right) \vh_l ^{(i)} + \sigmoid \left(\vw^\top \vz_{l+1} + b \right) \mW_B \vz_{l+1} (\vz_{l+1})_i\\
        & = \left(1 -  \sigmoid \left(\vw^\top \vz_{l+1} + b \right)\right) \sum_{j=1}^l G_{j,l}(\mZ) \vz_j[i] \mW_B \vz_j + G_{l+1,l+1}(\mZ) \vz_{l+1}[i] \mW_B \vz_{l+1}\\
        &=  \sum_{j=1}^{l+1} G_{j,l+1}(\mZ) \vz_j[i] \mW_B \vz_j
    \end{align*}
    and 
    \begin{align*}
        \vo_{l+1}[i]
        &= \left( \mW_C \vz_{l+1} \right)^\top \vh_{l+1}^{(i)}\\
        &= \left( \mW_C \vz_{l+1} \right)^\top  \sum_{j=1}^{l+1} G_{j,l+1}(\mZ) \vz_j[i] \mW_B \vz_j\\
        &= \sum_{j=1}^{l+1} G_{j,l+1}(\mZ) \vz_j[i] \vz_j^\top \mW_B^\top \mW_C \vz_{l+1}.
    \end{align*}
    Therefore, we have the desired conclusions.
\end{proof}

\subsection{High Probability Events}
Throughout the proof, we use the term ``with high probability'' which is defined as follows.
\begin{definition}
    We call that an event $E$ occurs \emph{with high probability}, when 
    \begin{equation*}
        \mathbb{P}[E] \geq 1 - d^{-C_{\mathrm{whp}}}
    \end{equation*}
    with a large enough $C_\mathrm{whp}>0$.
\end{definition}

For example, $\rvz = \bigO(\sqrt{ \log d})$ for $\rvz \sim \gN(0, 1)$, with high probability, which is a direct consequence of Gaussian tail bound. In addition, the intersection of a $\poly(d)$ events also occurs with high probability. We use these property frequently throughout our proof.

The following lemma is useful when we bound some quantities with high probability.
\begin{lemma}[Corollary 17 in \citet{oko2024pretrained}, adapted]\label{lemma:poly_bound}
Let $P$ be a polynomial with degree $\mathrm{deg}(P)$. If $\lVert \vbeta \rVert = 1$ and $\vx \sim \gN(\vzero, \mI_d)$, then $\abs{P (\langle \vbeta, \vx \rangle )} \lesssim (\log d)^{\deg(P)/2}$ holds with high probability.
\end{lemma}
This lemma implies that $y_i^t = \bigOtilde(1), \norm{\phi(\vx_i^t)}, \norm{\phi(\vx^t)} = \bigOtilde(d)$ holds for any $i \in [N_\pt], t \in [T_\pt]$, with high probability. We utilize these properties frequently in our proof.

The following lemma provides a high-probability guarantee regarding our input embedding, which is crucial for our analysis.
{
\begin{lemma}\label{lemma:hermite_concentration}
Let $(\rvx_i,\rvx_i',\rvz_i)_{i=1}^N$ be i.i.d. tuples whose first two coordinates have standard Gaussian marginals; $\rvz_i$ may depend on these coordinates. Suppose
$\mathbb E[\rvz_i^4]=O(1)$ and $|\rvz_i|\leq C$ with high probability, where $C$ is bounded above and below by fixed powers of $d$.
If {$1\leq N\leq N^*$}, then for $k=0,1,2$,
\begin{equation*}
 \left|\frac1N\sum_{i=1}^N\rvz_i\He_k(\rvx_i)
       -\mathbb E[\rvz_1\He_k(\rvx_1)]\right|
 =\bigOtilde(CN^{-1/2})
\end{equation*}
with high probability. The same bound holds with $\He_k(\rvx_i)$ replaced by $\rvx_i\rvx_i'$.
\end{lemma}
\begin{proof}[Proof of Lemma~\ref{lemma:hermite_concentration}]
We define $\widehat z_i=\rvz_i\mathbbm{1}[|\rvz_i|\leq C]$.
The products $\widehat z_i\He_k(\rvx_i)$ and
$\widehat z_i\rvx_i\rvx_i'$ have subexponential norms $O(C)$, since the multiplier is bounded and Gaussian coordinates and their quadratic products have bounded subexponential norms.
By applying Bernstein's inequality and using $N\leq d^{C^*}$, we have
\begin{equation*}
 \left|\frac1N\sum_{i=1}^N\widehat z_i\He_k(\rvx_i)
       -\mathbb E[\widehat z_1\He_k(\rvx_1)]\right|
  ={O\!\left(C\sqrt{\frac{\log d}{N}}+\frac{C\log d}{N}\right)}
 =\bigOtilde(CN^{-1/2}).
\end{equation*}
By a union bound, we can replace all $\widehat z_i$ in the sample average by $\rvz_i$ with high probability.
For the expectation, by applying the Cauchy--Schwarz inequality twice, we have
\begin{align*}
 |\mathbb E[(\rvz_1-\widehat z_1)\He_k(\rvx_1)]|
 &\leq
  \big(\mathbb E[\rvz_1^2\mathbbm{1}[|\rvz_1|>C]]\big)^{1/2}
  \big(\mathbb E[\He_k(\rvx_1)^2]\big)^{1/2}\\
 &\leq(\mathbb E[\rvz_1^4])^{1/4}
  \mathbb P[|\rvz_1|>C]^{1/4}
  \big(\mathbb E[\He_k(\rvx_1)^2]\big)^{1/2}.
\end{align*}
By choosing the high-probability exponent sufficiently large, we make this term negligible relative to $CN^{-1/2}$, since $C^{-1}$ and $N$ are polynomially bounded.
We apply the same argument to $\rvx_1\rvx_1'$. Its second moment is bounded by the Cauchy--Schwarz inequality, which does not require independence of the two coordinates within a tuple.
\end{proof}

}
\subsection{Reducing the Information Exponent with Label Transformation}

For any {polynomial $h$} and $p \in \mathbb{N}\cup \{0\}$, we define $e_p(h) \coloneqq \min \{i \in \mathbb{N} : H(h^i,p) \neq 0 \}$.
If this minimum does not exist (i.e., the set is empty), we set $e_p(h) = \infty$. 
In addition, we define 
\begin{align*}
    \bar g_* (z) := 
    \begin{cases}
        \frac{g_*(z)}{\rho} &\text{if } \abs{\frac{g_*(z)}{\rho}} \leq \frac{1}{ \log d}\\
        0 &\text{otherwise}
    \end{cases}.
\end{align*}

By our choice of $\rho$, we have $\bar g_*(\langle\vbeta^t,\vx_i^t\rangle)=g_*(\langle\vbeta^t,\vx_i^t\rangle)/\rho$ for all pretraining context examples with high probability. We use this identity in the following analysis and treat the independent noise $\zeta_i^t$ separately.

We also define the following function, which naturally appears in our analysis:
\begin{equation*}
    A(z):= \frac{1}{2} \big[( \rho \bar g_*(z)+ \tau) \sigma\left( \bar g_*(z) + \tau/\rho + b\right) + ( \rho \bar g_*(z)- \tau) \sigma\left( \bar g_*(z) - \tau/\rho + b\right)\big].
\end{equation*}
Let $A(z) = \sum_{k=0}^\infty \frac{a_k}{k!} \He_k(z)$ be the Hermite expansion. The following lemma characterizes its Hermite coefficients.
{
\begin{lemma}\label{lemma:sigmoid_ie}
Let $k=\deg(g_*)$ and choose $C_\rho>1+k/2$. For any fixed $p\geq1$ such that $e=e_p(g_*)<\infty$,
\begin{equation}\label{eq:hermite_coefficient_scale}
 a_p=\frac{\sigmoid^{(e-1)}(b)}{(e-1)!\rho^{e-1}}H(g_*^e,p)
       +O(d^{-C_b}\rho^{-e}).
\end{equation}
Consequently, $d^{C_b}|a_p|=\Theta(\rho^{-(e-1)})$. The constant coefficient satisfies
\begin{equation*}
 a_0=\frac{\sigmoid'(b)}{\rho}(1+\tau^2)
       +O(d^{-C_b}\rho^{-2}),
 \qquad d^{C_b}a_0=\Theta(\rho^{-1}).
\end{equation*}
The constants may depend on the fixed link, $p$, and $\tau$. If $g_*$ is even, then $A$ is even and $a_1=0$ exactly. The indices needed below satisfy $e_0(g_*)=2$, $e_2(g_*)<\infty$, and $e_1(g_*)<\infty$ whenever $g_*$ is not even.
\end{lemma}

}
The following two lemmas are crucial for our proof of Lemma~\ref{lemma:sigmoid_ie}.
\begin{lemma}[Proposition 6 in \citet{lee2024neural}]\label{lemma:monomial}
    For any {nonconstant} polynomial $P$, there exist $C_P,D_P>0$ depending only on $P$ such that the following holds.
    \begin{itemize}
        \item If $P$ is not an even function, then there exists $i \leq C_P$ such that $\abs{H(P^i,1)} \geq D_P$.
        \item If $P$ is an even function, then there exists $i \leq C_P$ such that $\abs{H(P^i, 2)} \geq D_P$.
    \end{itemize}
\end{lemma}

\begin{lemma}\label{lemma:sigmoid_diff}
For any $k \in \mathbb{N} \cup \{0\}$ and $z < - k - 2 $, we have $\frac{e^{z}}{2}\leq \sigmoid^{(k)}(z) \leq 2 e^z$.
\end{lemma}
\begin{proof}[Proof of Lemma~\ref{lemma:sigmoid_diff}]
    For any $x<0$, we have
    \begin{equation*}
        \sigmoid(x) = \frac{1}{1 + \exp(-x)} = 1 - \frac{1}{1 + \exp(x)} = \sum_{j=1}^{\infty} (-1)^{j-1} e^{jx}.
    \end{equation*}
    Therefore, we have
    \begin{equation*}
        \sigmoid^{(k)}(z) = \sum_{j=1}^{\infty} (-1)^{j-1} j^k e^{jz} = e^{z} + \sum_{j=2}^{\infty} (-1)^{j-1} j^k e^{jz}.
    \end{equation*}
    For each $j \geq 2$, since $\frac{(j+1)^k e^{(j+1)z}}{j^k e^{jz}} \leq 2^k e^z$ and $2^k e^z <\frac{1}{3}$, we have
    \begin{equation*}
        \abs{\sigmoid^{(k)}(z) - e^z} \leq \sum_{j=2}^\infty j^k e^{jz} \leq 2^k e^{2z} \sum_{j=0}^\infty \left( 2^k e^z \right)^j = \frac{2^k e^{2z}}{1- 2^k e^z} \leq \frac{e^z}{2}. 
    \end{equation*}
    Hence, we have a desired conclusion.
\end{proof}

We now prove Lemma~\ref{lemma:sigmoid_ie}.
{
\begin{proof}[Proof of Lemma~\ref{lemma:sigmoid_ie}]

Let $Z\sim\gN(0,1)$ and let $\zeta$ be independent and uniform on $\{\pm\tau\}$. Let
\begin{equation*}
 E_\rho:=\{|g_*(Z)|>\rho/\log d\}.
\end{equation*}
Then $A(z)=\mathbb E_\zeta[(\rho\bar g_*(z)+\zeta)
\sigmoid(b+(\rho\bar g_*(z)+\zeta)/\rho)]$. Since $g_*$ has fixed degree $k$ and $C_\rho>1+k/2$, the Gaussian tail bound gives, for every fixed $j,p,M$,
\begin{equation}\label{eq:weighted_truncation_tail}
 \mathbb E[(1+|g_*(Z)|^j)|\He_p(Z)|\mathbbm{1}[E_\rho]]
 =O(d^{-M}).
\end{equation}
To show this, note that $E_\rho$ implies $|Z|\geq c(\log d)^{(C_\rho-1)/k}$ for sufficiently large $d$. By applying the Cauchy--Schwarz inequality, we bound this expectation by a fixed Gaussian polynomial moment times $\mathbb P[E_\rho]^{1/2}$, which decays faster than every power of $d^{-1}$.

For any fixed integer $J\geq0$, by applying Taylor's theorem to the sigmoid at $b$, we have
\begin{equation*}
 A(z)=\sum_{j=0}^{J}
 \frac{\sigmoid^{(j)}(b)}{j!\rho^j}
       \mathbb E_\zeta[(\rho\bar g_*(z)+\zeta)^{j+1}]
       +R_J(z).
\end{equation*}
We have $|(\rho\bar g_*(z)+\zeta)/\rho|\leq2/\log d$ for large $d$. By Lemma~\ref{lemma:sigmoid_diff}, we obtain
\begin{equation*}
 |R_J(z)|\leq C_J d^{-C_b}\rho^{-(J+1)}
                 \mathbb E_\zeta[|\rho\bar g_*(z)+\zeta|^{J+2}],
 \qquad
 |H(R_J,p)|=O(d^{-C_b}\rho^{-(J+1)}).
\end{equation*}
The moments in this bound are finite and uniformly bounded in $d$. By \eqref{eq:weighted_truncation_tail}, we can replace each $\rho\bar g_*$ in the finite sum by $g_*$ with an error $O(d^{-C_b-M})$ for arbitrary fixed $M$. Therefore, the truncation error also contains the factor $d^{-C_b}$.

For $p\geq1$ and $e=e_p(g_*)<\infty$, we choose $J=e-1$. By expanding $(g_*+\zeta)^{j+1}$ and using independence, we find that its $p$th Hermite coefficient vanishes for $j+1<e$, since all positive powers of $g_*$ below $e$ have zero coefficient and $H(1,p)=0$. At $j=e-1$, the only nonzero term is $H(g_*^e,p)$. Thus, we obtain \eqref{eq:hermite_coefficient_scale}. By the series expansion in the proof of Lemma~\ref{lemma:sigmoid_diff}, we also have
$\sigmoid^{(j)}(b)=d^{-C_b}(1+O(d^{-C_b}))$ for fixed $j$, which gives the stated magnitude bound without a sign assumption on $a_p$.

For $p=0$, take $J=1$ and use
$\mathbb E[g_*(Z)+\zeta]=0$ and
$\mathbb E[(g_*(Z)+\zeta)^2]=1+\tau^2$.
This gives the claimed expression for $a_0$, including the noise contribution.

Finally, $e_0=2$ follows from the normalization of $g_*$. By Lemma~\ref{lemma:monomial}, we have finite $e_1$ for non-even links. For even links, $a_1=0$ follows directly from evenness and does not use a truncation approximation. To show that $e_2<\infty$ for every nonconstant polynomial, we let $M_0=\max_{|z|\leq1}|g_*(z)|$. We choose a fixed interval $I$ outside $[-1,1]$ on which $|g_*(z)|>2M_0$ and $\He_2(z)>0$. For sufficiently large integers $j$, the positive contribution of $I$ to $H(g_*^{2j},2)$ exceeds the negative contribution from $[-1,1]$, and all other contributions outside $[-1,1]$ are nonnegative. Therefore, $H(g_*^{2j},2)>0$ for some finite $j$.

\end{proof}
}

\section{One-Step Gradient Descent on the Mamba Layer}\label{appendix:mamba}

Let us define the function $\psi: \R^{d+3} \rightarrow \R^{\tilde d}$, which we use repeatedly in our proof. It is defined as
\begin{equation*}
     \psi(\bm{\theta}, c_0, c_1, c_2) \coloneqq \left[c_0, c_1 \bm{\theta}^\top, \frac{c_2 \left(\bm{\theta} \odot \bm{\theta}\right)^\top}{\sqrt{2}}, c_2 \bm{\theta}[1] \bm{\theta}[2], \dots, c_2 \bm{\theta}[d-1] \bm{\theta}[d] \right]^\top.
\end{equation*}
Note that for any vector $\bm{\theta} \in \R^d$, 
\begin{equation}\label{eq:v_norm}
    \norm{\psi(\bm{\theta}, c_0, c_1, c_2)}^2 = c_0^2 + c_1^2 \norm{\bm{\theta}}^2 + \frac{c_2^2 \norm{\bm{\theta}}^4}{2}.
\end{equation}

If we choose $\lambda_1 = \eta^{-1}$, the updated parameter $\vgamma^*$ can be expressed as 
\begin{align*}
    \vgamma^* &= \frac{2\eta}{T_1} \sum_{t \in [T_1]}\left[ \Big(y^t-f\left(\mZ^t; \vgamma^{(0)}, \vu^{(0)}, \vv^{(0)}, \va^{(0)}\right)\Big)\nabla_\vgamma f\left(\mZ^t; \vgamma^{(0)}, \vu^{(0)}, \vv^{(0)}, \va^{(0)}\right)\right]\\
    &= \frac{2\eta}{T_1} \sum_{t \in [T_1]} y^t \nabla_\vgamma f\left(\mZ^t; \vgamma^{(0)}, \vu^{(0)}, \vv^{(0)}, \va^{(0)}\right)\\
    &\quad - \frac{2 \eta}{T_1}  \sum_{t \in [T_1]} {f\left(\mZ^t; \vgamma^{(0)}, \vu^{(0)}, \vv^{(0)}, \va^{(0)}\right)} \nabla_\vgamma f\left(\mZ^t; \vgamma^{(0)}, \vu^{(0)}, \vv^{(0)}, \va^{(0)}\right).
\end{align*}

The initial output evaluated at $\mZ^t$ can be bounded as
\begin{align*}
    \abs{f\left( \mZ^t; \vgamma^{(0)}, \vu^{(0)}, \vv^{(0)}, \va^{(0)} \right)}
    &=  \relu \left( N_\pt^{-1} \sum_{j=1}^{N_\pt} {G_{j,N_\pt+1}(\mZ^t)} y_j^t \phi\left (\vx_j^t \right)^\top \left(\vgamma^{(0)} \odot \phi\left (\vx^t \right)\right) \right)\\
    &\leq  N_\pt^{-1} \sum_{j=1}^{N_\pt} {G_{j,N_\pt+1}(\mZ^t)} \abs{y_j^t} \norm{\phi \left( \vx_j^t \right)} \norm{\phi \left( \vx^t \right)}\\
    &= \bigOtilde\left(d^{-C_b+2} \right),
\end{align*}
with high probability.

The gradient of $\vgamma$ of output evaluated at $\mZ^t$ can be calculated as 
\begin{align*}
    &\quad \nabla_\vgamma f\left(\mZ^t; \vgamma^{(0)}, \vu
    ^{(0)}, \vv^{(0)}, \va^{(0)}\right)\\
    &= \mathbbm{1} \left[ \sum_{j=1}^{N_\pt}   G_{j, N_\pt+1}\left(\mZ^t\right) y_j^t \phi \left(\vx_j^t\right) ^\top \left(\vgamma^{(0)} \odot \phi \left( \vx^t \right)\right) >0 \right] \\
    &\quad \times \left( N_\pt^{-1} \sum_{j=1}^{N_\pt}   G_{j, N_\pt+1}\left(\mZ^t\right) y_j^t \phi \left( \vx_j^t \right) \odot \phi \left(\vx^t\right) \right)
\end{align*}
and its norm can be bounded as
\begin{align*}
    \norm{ \nabla_\vgamma f\left( \mZ^t; \vgamma^{(0)}, \vu^{(0)}, \vv^{(0)}, \va^{(0)} \right)}
    &\leq  N_\pt^{-1} \sum_{j=1}^{N_\pt} {G_{j,N_\pt+1}(\mZ^t)} \abs{y_j^t} \norm{\phi \left( \vx_j^t \right)} \norm{\phi \left( \vx^t \right)}\\
    &= \bigOtilde\left( d^{-C_b+2} \right),
\end{align*}
with high probability. Therefore, with high probability, we have
\begin{equation*}
    \vgamma^* = \frac{2\eta}{T_1} \sum_{t \in [T_1]} y^t \nabla_\vgamma f\left(\mZ^t; \vgamma^{(0)}, \vu^{(0)}, \vv^{(0)}, \va^{(0)}\right) + \bigOtilde \left (\eta   d^{-2C_b+4} \right).
\end{equation*}
Hence, we will focus on estimating the first term.

\subsection{Estimation of Label-Gradient Correlation}

We first establish a high-probability guarantee for the term containing context examples.

\begin{lemma}\label{lemma:vector_concentration} 
    Let $(\vx_1, y_1, \dots, \vx_N,y_N, \vx)$ be a prompt with context length $N \leq N^*$ and feature vector {$\vbeta \in S_r$} and its embedding $\mZ\in \R^{(\tilde d +1) \times (N+1)}$. Then, the following holds with high probability:
    \begin{equation*}
         N^{-1} \sum_{j=1}^{N} G_{j, N +1} \left(\mZ\right) y_j \phi \left( \vx_j \right) =  \psi \left( \vbeta, a_0, a_1, a_2\right) + \bigOtilde \left( d^{- C_b + 1} N^{-1/2}\right).
    \end{equation*}
\end{lemma}

{
\begin{proof}[Proof of Lemma~\ref{lemma:vector_concentration}]

Recall that, with high probability,
$y_j/\rho=\bar g_*(\langle\vbeta,\vx_j\rangle)+\zeta_j/\rho$
simultaneously for all $j\in[N]$. On this event,
\begin{align*}
 &\quad \frac1N\sum_{j=1}^N
 y_j\sigmoid(y_j/\rho+b)\phi(\vx_j)\\
 &=\frac1N\sum_{j=1}^N
 \left(\rho\bar g_*(\langle\vbeta,\vx_j\rangle)+\zeta_j\right)
 \sigmoid\left(\bar g_*(\langle\vbeta,\vx_j\rangle)+\zeta_j/\rho+b\right)
 \phi(\vx_j).
\end{align*}
For fixed $\vbeta$, the truncated summands on the right are independent.
After multiplication by $d^{C_b}$, their scalar multipliers have bounded fourth moments and are bounded by a polylogarithm with high probability.
By the definition of $A$ and Gaussian conditional expectation for the first two Hermite degrees,
\begin{align*}
 &\quad\mathbb E\left[
 \left(\rho\bar g_*(\langle\vbeta,\vx_j\rangle)+\zeta_j\right)
 \sigmoid\left(\bar g_*(\langle\vbeta,\vx_j\rangle)+\zeta_j/\rho+b\right)
 \phi(\vx_j)\right]\\
 &=\mathbb E\left[A(\langle\vbeta,\vx_j\rangle)\phi(\vx_j)\right]
 =\psi(\vbeta,a_0,a_1,a_2).
\end{align*}
Applying Lemma~\ref{lemma:hermite_concentration} to these truncated summands coordinatewise, and taking a union bound over the $O(d^2)$ embedding coordinates, therefore gives
\begin{equation*}
 \left\|\frac1N\sum_{j=1}^N
 y_j\sigmoid(y_j/\rho+b)\phi(\vx_j)-\psi(\vbeta,a_0,a_1,a_2)\right\|
 =\bigOtilde(d^{-C_b+1}N^{-1/2})
\end{equation*}
with high probability.

In addition, all gate probabilities are $O(d^{-C_b})$ with high probability, and
\begin{equation*}
 0\leq1-(1-\sigmoid(b))
 \prod_{i=j+1}^N(1-\sigmoid(y_i/\rho+b))
 \leq C(N+1)d^{-C_b}.
\end{equation*}
Using $|y_j|=\bigOtilde(1)$ and $\|\phi(\vx_j)\|=\bigOtilde(d)$, we obtain
\begin{align*}
 &\quad\left\|\frac1N\sum_{j=1}^N G_{j,N+1}(\mZ)y_j\phi(\vx_j)
 -\frac1N\sum_{j=1}^N y_j\sigmoid(y_j/\rho+b)\phi(\vx_j)\right\|\\
 &\leq\frac1N\sum_{j=1}^N |y_j\sigmoid(y_j/\rho+b)|
 \left(1-(1-\sigmoid(b))\prod_{i=j+1}^N(1-\sigmoid(y_i/\rho+b))\right)
 \|\phi(\vx_j)\|\\
 &=\bigOtilde(Nd^{-2C_b+1}).
\end{align*}
For $N\leq d^{C^*}$ and $C_b>3C^*/2$, this is absorbed into
$\bigOtilde(d^{-C_b+1}N^{-1/2})$.
The triangle inequality completes the proof.

\end{proof}

}
\begin{corollary}\label{corollary:inner_concentration}
    For each $t \in [T_1]$, the following holds with high probability:
    \begin{align*}
        &\quad N_\pt^{-1}\sum_{j=1}^{N_\pt}G_{j, N_\pt +1 }\left( \mZ^t\right)  y_j^t  \phi \left( \vx_j^t \right)^\top \left( \vgamma^{(0)} \odot \phi \left( \vx^t \right)\right) \\
        &=   a_0 \gamma^2 + a_1 \left \langle \vbeta^t , \vx^t \right \rangle + \frac{1}{2} a_2 \gamma \He_2 \left( \left \langle \vbeta^t, \vx^t \right \rangle \right) + \bigOtilde \left( d^{-C_b+2} N_\pt^{- 1/2}\right).
    \end{align*}  
\end{corollary}
\begin{proof}[Proof of Corollary~\ref{corollary:inner_concentration}]
    From Lemma~\ref{lemma:vector_concentration}, for each $t \in [T_1]$, with high probability, we have
    \begin{align*}
        &\quad N_\pt^{-1}\sum_{j=1}^{N_\pt}G_{j, N_\pt +1 }\left( \mZ^t\right)  y_j^t  \phi \left( \vx_j^t \right)^\top \left( \vgamma^{(0)} \odot \phi \left( \vx^t \right) \right)\\
        &= a_0 \gamma^2 + a_1 \left \langle \vbeta^t , \vx^t \right \rangle + \frac{1}{2} a_2 \gamma \He_2 \left( \left \langle \vbeta^t, \vx^t \right \rangle \right) + \bigOtilde \left( d^{-C_b+1} N_\pt^{- 1/2}\right) \norm{ \phi \left( \vx^t \right)} \\
        &= a_0 \gamma^2 + a_1 \left \langle \vbeta^t , \vx^t \right \rangle + \frac{1}{2} a_2 \gamma \He_2 \left( \left \langle \vbeta^t, \vx^t \right \rangle \right) + \bigOtilde \left( d^{-C_b+2} N_\pt^{- 1/2}\right),
    \end{align*} 
    where we use Lemma~\ref{lemma:poly_bound} and \eqref{eq:v_norm} for the last equality.
\end{proof}

Next, let us estimate the expectation of the first term that appears in the label-gradient correlation. The following lemma is useful for this purpose.

\begin{lemma}\label{lemma:small_abs}
    For any $\delta >0$ with $\delta = \bigOtilde(d^{-C_\delta})$ for some constant $C_\delta>0$, the following holds:
    \begin{equation*}
        \mathbb{P}_{\rvz \sim \gN(0,1)} \left[ \abs{a_0\gamma^2 + a_1 \rvz + \frac{1}{2}a_2 \gamma \He_2 \left( \rvz \right) }< d^{-C_b}\delta\right] \leq \bigOtilde\left ( d^{- C_\delta} \right) .
    \end{equation*}
\end{lemma}

{
\begin{proof}[Proof of Lemma~\ref{lemma:small_abs}]
Let $\alpha_i=d^{C_b}a_i$ and
$h(z)=\alpha_0\gamma^2+\alpha_1z+\alpha_2\gamma(z^2-1)/2$.
Lemma~\ref{lemma:sigmoid_ie} gives
$\alpha_0=\Theta(\rho^{-1})$ and
$|\alpha_p|=\Theta(\rho^{-(e_p(g_*)-1)})$ whenever $p=1,2$ and $e_p(g_*)<\infty$.
For even links, $\alpha_1=0$ exactly.

If $\alpha_2=0$, then $\alpha_1\ne0$, and the event $|h(Z)|<\delta$ is an interval of length $2\delta/|\alpha_1|$. Its Gaussian probability is $\bigOtilde(\delta)$.
Otherwise, we define
\begin{equation*}
 D:=\alpha_1^2+\alpha_2^2\gamma^2-2\alpha_0\alpha_2\gamma^3.
\end{equation*}
For non-even links, we choose $C_\gamma$ sufficiently large that
$|\alpha_2|\gamma=o(|\alpha_1|)$ and
$|\alpha_0\alpha_2|\gamma^3=o(\alpha_1^2)$; then $D\asymp\alpha_1^2$.
For even links, we choose $|\alpha_0|\gamma=o(|\alpha_2|)$; then
$D\asymp\alpha_2^2\gamma^2$.
In both cases $D>0$ and $\sqrt D$ is bounded below by an inverse power of $\log d$.

We obtain the endpoints of $\{|h(z)|<\delta\}$ from the roots of $h(z)=\pm\delta$, whose discriminants are $D\pm2\alpha_2\gamma\delta$. Since $\delta$ is polynomially small,
$|\alpha_2|\gamma\delta=o(D)$. By rationalizing the differences of these roots, we bound the total length of the intervals by
\begin{equation*}
 \frac{C\delta}
 {\sqrt{D+2\alpha_2\gamma\delta}+\sqrt{D-2\alpha_2\gamma\delta}}
 =O(\delta/\sqrt D)=\bigOtilde(\delta).
\end{equation*}
We use the absolute lengths of the intervals, so the bound holds for either sign of $\alpha_2$. Since the standard Gaussian density is bounded, the same bound holds for the probability.
\end{proof}

}
{
For each $t\in[T_1]$, let $E_t$ be the event that
\begin{align*}
 &\quad \mathbbm{1}\!\left[
 N_\pt^{-1}\sum_{j=1}^{N_\pt}
 G_{j,N_\pt+1}(\mZ^t)y_j^t
 \phi(\vx_j^t)^\top(\vgamma^{(0)}\odot\phi(\vx^t))>0\right]\\
 &\ne
 \mathbbm{1}\!\left[
 a_0\gamma^2+a_1\langle\vbeta^t,\vx^t\rangle
 +\frac12a_2\gamma\He_2(\langle\vbeta^t,\vx^t\rangle)>0\right].
\end{align*}

}
From Corollary~\ref{corollary:inner_concentration} and Lemma~\ref{lemma:small_abs}, we have
\begin{align*}
    \mathbb{P}_{\mZ^t} [E_t]
    &\leq\mathbb{P}_{\mZ^t} \left[ \abs{a_0 \gamma^2 + a_1 \left \langle \vbeta^t, \vx^t \right \rangle + \frac{1}{2} a_2 \gamma \He_2 \left( \left \langle \vbeta^t, \vx^t \right \rangle \right)} < \bigOtilde \left( d^{-C_b +2} N_\pt^{-1/2}\right) \right] + \frac{1}{\poly (d)}\\
    &=  \bigOtilde \left( d^{2} N_\pt^{- 1/2}\right) .
\end{align*}

Combining with Corollary~\ref{corollary:inner_concentration}, with probability at least $1- \bigOtilde \left( d^{2} T_1 N_\pt^{- 1/2}\right)$ the following holds: For any $t \in [T_1]$, we have
\begin{align*}
    &\quad  y^t \nabla_\vgamma f\left(\mZ^t; {\vgamma^{(0)}, \vu^{(0)}, \vv^{(0)}, \va^{(0)}}\right)\\
    &= y^t \mathbbm{1} \left[N_\pt^{-1}\sum_{j=1}^{N_\pt}G_{j, N_\pt +1 }\left( \mZ^t\right)  y_j^t  \phi \left( \vx_j^t \right)^\top \left({\vgamma^{(0)}} \odot \phi \left( \vx^t \right)\right) > 0 \right]\\
    &\quad \times  \left( N_\pt^{-1} {\sum_{j=1}^{N_\pt}} G_{j, N_\pt +1} \left( \mZ^t \right) y_j^t \phi \left( \vx_j^t \right) \odot \phi \left( \vx^t \right) \right)\\
    &=   y^t \mathbbm{1} \left[a_0 \gamma^2 + a_1 \left \langle \vbeta^t , \vx^t \right \rangle + \frac 1 2 a_2 \gamma \He_2\left( \left \langle \vbeta^t, \vx^t \right \rangle \right) > 0 \right] \phi(\vx^t) \odot   \psi \left( \vbeta^t, a_0, a_1, a_2\right) \\
    &\quad +   y^t \mathbbm{1} \left[a_0 \gamma^2 + a_1 \left \langle \vbeta^t , \vx^t \right \rangle + \frac 1 2 a_2 \gamma \He_2\left( \left \langle \vbeta^t, \vx^t \right \rangle \right) > 0 \right] \phi(\vx^t) \odot  \vn \left(\mZ^t\right),
\end{align*}
where $\norm{\vn(\mZ^t)} = \bigOtilde \left( d^{-C_b + 1 } N_\pt^{-1/2}\right)$ with high probability. With high probability, the following holds for all $t \in [T_1]$:
\begin{align*}
    &\quad \norm{   y^t \mathbbm{1} \left[a_0 \gamma^2 + a_1 \left \langle \vbeta^t, \vx^t \right \rangle + \frac 1 2 a_2 \gamma \He_2\left( \left \langle \vbeta^t, \vx^t\right \rangle \right) > 0 \right]  \phi \left( \vx^t \right) \odot \vn(\mZ^t) }\\
    &\leq \abs{y^t} \norm{\vn\left( \mZ^t \right)} \norm{\phi\left( \vx^t \right) }\\
    &= \bigOtilde \left( d^{-C_b+2} N_{\pt}^{-1/2} \right).
\end{align*}
\paragraph{Estimation of label-gradient correlation.} With probability at least $1- \bigOtilde \left( d^{2} T_1 N_\pt^{- \frac 1 2}\right)$, the following holds for all $t \in [T_1]$:
\begin{align*}
    &\quad y^t \nabla_\vgamma f\left(\mZ^t; {\vgamma^{(0)}, \vu^{(0)}, \vv^{(0)}, \va^{(0)}}\right)\\
    &=  y^t \mathbbm{1} \left[a_0 \gamma^2 + a_1 \left \langle \vbeta^t, \vx^t \right \rangle + \frac 1 2 a_2 \gamma \He_2\left( \left \langle \vbeta^t, \vx^t \right \rangle \right) > 0 \right]  \phi(\vx^t) \odot \psi \left( \vbeta^t, a_0, a_1, a_2\right) \\
    &\quad + \bigOtilde \left( d^{-C_b+2} N_\pt^{-1/2}\right).
\end{align*}

\subsection{Characterization of  Updated Parameter}
In this step, we characterize the updated parameter by establishing concentration results.

{
Every entry of $\psi(\vbeta^t,a_0,a_1,a_2)$ is $\bigOtilde(d^{-C_b})$.
By applying Lemma~\ref{lemma:hermite_concentration} coordinatewise to the independent tasks, we obtain, with high probability,
\begin{align*}
 &\left\|\frac1{T_1}\sum_{t=1}^{T_1}
 y^t\mathbbm{1}\!\left[
 a_0\gamma^2+a_1\langle\vbeta^t,\vx^t\rangle+
 \frac12a_2\gamma\He_2(\langle\vbeta^t,\vx^t\rangle)>0\right]
 \phi(\vx^t)\odot\psi(\vbeta^t,a_0,a_1,a_2)-\vc\right\|\\
 &\hspace{3cm}=\bigOtilde(d^{-C_b+1}T_1^{-1/2}),
\end{align*}
where $\vc$ is the expectation of a summand.
We define
\begin{equation*}
 B(z)=g_*(z)\mathbbm{1}[a_0\gamma^2+a_1z+\tfrac12a_2\gamma\He_2(z)>0],
 \qquad B(z)=\sum_{j=0}^\infty\frac{b_j}{j!}\He_j(z).
\end{equation*}
The query noise is centered and independent of the activation indicator. Thus, by Gaussian conditional expectation, we have
\begin{align*}
 \vc&=\mathbb E_{\vbeta\sim\mathrm{Unif}(S_r)}
 \left[\psi(\vbeta,a_0,a_1,a_2)\odot
 \mathbb E_{\vx\sim\gN(0,\mI_d)}
 [B(\langle\vbeta,\vx\rangle)\phi(\vx)]\right]\\
 &=\mathbb E_{\vbeta\sim\mathrm{Unif}(S_r)}
 [\psi(\vbeta,a_0,a_1,a_2)\odot\psi(\vbeta,b_0,b_1,b_2)].
\end{align*}

}
{
Therefore, we conclude that
\begin{align*}
 &\quad \frac1{T_1}\sum_{t=1}^{T_1}
 y^t\nabla_\vgamma f(\mZ^t;\vgamma^{(0)},\vu^{(0)},\vv^{(0)},\va^{(0)})\\
 &=\mathbb E_{\vbeta\sim\mathrm{Unif}(S_r)}
 [\psi(\vbeta,a_0,a_1,a_2)\odot\psi(\vbeta,b_0,b_1,b_2)]+\bigOtilde(d^{-C_b+2}N_\pt^{-1/2})
       +\bigOtilde(d^{-C_b+1}T_1^{-1/2}),
\end{align*}

}
with probability at least $1- \bigOtilde \left( d^{2} T_1 N_\pt^{-1/2}\right)$.

The remaining step is to characterize the Hermite coefficients $b_0,b_1, b_2$, and the following lemma is useful.

\begin{lemma}\label{lemma:activation_hermite}
 For any non zero polynomial $P$ independent of $r$ and $d$, the following holds except for the cases $g_*$ is even function and $P$ is an odd function:
 \begin{equation*}
     \abs{\mathbb{E}_{\rvz \sim \gN(0,1)} \left[ P(\rvz) \mathbbm{1} \left [a_0 \gamma^2 + a_1 \rvz + \frac 1 2 a_2 \gamma \He_2 (\rvz)>0\right ]\right]} \gtrsim \frac{1}{\polylog(d)} .
 \end{equation*}
 Here, dependency of $g_*$ appears in $a_0, a_1, a_2$.
\end{lemma}

{
\begin{proof}[Proof of Lemma~\ref{lemma:activation_hermite}]
Let $\varphi(z)=(2\pi)^{-1/2}e^{-z^2/2}$ and
$h(z)=a_0\gamma^2+a_1z+a_2\gamma(z^2-1)/2$.
By Lemma~\ref{lemma:sigmoid_ie}, all nonzero coefficients used below, after multiplication by $d^{C_b}$, are bounded above and below in magnitude by powers of $\log d$. In particular, $a_0>0$ and $e_2(g_*)<\infty$. We choose $C_\gamma$ sufficiently large relative to these exponents.

\textbf{Case 1: $g_*$ is not an even function.}
Here $a_1\ne0$. By substituting $z\mapsto-z$ and replacing $P$ accordingly, we may assume $a_1>0$. If $a_2=0$, the event $h(z)>0$ is $z>t$, where $t=-a_0\gamma^2/a_1$. This threshold tends to zero and its magnitude is bounded below by an inverse power of $\log d$.

If $a_2\ne0$, set
\begin{equation*}
 D=a_1^2+a_2^2\gamma^2-2a_0a_2\gamma^3,\qquad
 t=\frac{-a_1+\sqrt D}{a_2\gamma}
   =\frac{(a_2-2a_0\gamma)\gamma}{a_1+\sqrt D}.
\end{equation*}
We choose $\gamma$ so that $|a_0|\gamma=o(|a_2|)$ and
$|a_2|\gamma=o(|a_1|)$. Then
\begin{equation*}
 D\asymp a_1^2,\qquad
 |t|\asymp |a_2|\gamma/|a_1|=o(1).
\end{equation*}
By the coefficient bounds, $|t|$ is bounded below by an inverse power of $\log d$, in addition to the upper bound on $|t|$.
The other root is
$t_{\mathrm{far}}=(-a_1-\sqrt D)/(a_2\gamma)$.
For sufficiently large $C_\gamma$, we have $|t_{\mathrm{far}}|\geq(\log d)^2$. Since $h'(t)=\sqrt D>0$, the indicators of $h(z)>0$ and $z>t$ differ only beyond this root for either sign of $a_2$. Hence, for every fixed $M>0$,
\begin{equation*}
 \mathbb E[P(Z)\mathbbm{1}[h(Z)>0]]
 =F(t)+O(d^{-M}),\qquad
 F(t):=\int_t^\infty P(z)\varphi(z)\,\mathrm dz.
\end{equation*}
The same identity holds without the tail error when $a_2=0$.
If $F(0)\ne0$, its magnitude stays bounded away from zero. If $F(0)=0$, write
$P(z)=c_\nu z^\nu+O(z^{\nu+1})$ with $c_\nu\ne0$. By integrating near the boundary, we have
\begin{equation*}
 F(t)=-\frac{c_\nu\varphi(0)}{\nu+1}t^{\nu+1}
       +O(|t|^{\nu+2}).
\end{equation*}
The lower bound on $|t|$ and the Gaussian tail bound give the desired inverse-polylogarithmic bound.

\textbf{Case 2: $g_*$ is an even function.}
Here $a_1=0$ exactly, $a_2\ne0$, and only the even part $P_{\mathrm e}$ of $P$ contributes. By the assumption on $P$, we have $P_{\mathrm e}\ne0$.
Set $t=(1-2a_0\gamma/a_2)^{1/2}$. With $|a_0|\gamma=o(|a_2|)$,
\begin{equation*}
 t\longrightarrow1,\qquad
 |t-1|\asymp |a_0|\gamma/|a_2|.
\end{equation*}
The last quantity is bounded below by an inverse power of $\log d$.
The required expectation is
\begin{equation*}
 F(t)=
 \begin{cases}
  2\int_t^\infty P_{\mathrm e}(z)\varphi(z)\,\mathrm dz,&a_2>0,\\
  2\int_0^t P_{\mathrm e}(z)\varphi(z)\,\mathrm dz,&a_2<0.
 \end{cases}
\end{equation*}
If $F(1)\ne0$, the conclusion follows by continuity. Otherwise, let $\nu$ be the multiplicity of the zero of $P_{\mathrm e}$ at $1$ (with $\nu=0$ if it is nonzero there). By Taylor expansion of the integrand, we have
$|F(t)|\asymp |t-1|^{\nu+1}$, which gives the desired lower bound for both signs.
\end{proof}

}
{
By applying Lemma~\ref{lemma:activation_hermite} to $P(z)=g_*(z)\He_j(z)$ for $j=0,1,2$, we obtain
$|b_0|,|b_2|=\tilde\Theta(1)$ and $|b_1|=\tilde\Theta(1)$ for non-even links. For even links, $B$ is even and $b_1=0$ exactly. These bounds hold for either sign of the nonzero coefficients.

}
\paragraph{Updated Parameter $\vgamma^*$.} With probability at least $1- \bigOtilde \left( d^2 T_1 N_\pt^{- 1/2}\right)$, the updated parameter $\vgamma^*$ is given by:
\begin{equation}\label{eq:updated_mamba}
\begin{aligned}
    \vgamma^* &= 2 \eta \mathbb{E}_{\vbeta \sim \mathrm{Unif}(S_r)}\left[\psi(\vbeta, a_0, a_1, a_2) \odot \psi (\vbeta, b_0, b_1, b_2) \right] \\
    &\quad + \eta \bigOtilde \left( \max \left \{ d^{-2C_b+4}, d^{-C_b+2}N_\pt^{-1/2}, d^{-C_b+1} T_1^{-1/2} \right \}\right),
\end{aligned}
\end{equation}
{
Here, we have $d^{C_b}a_0,d^{C_b}|a_{\ge(g_*)}|,|b_0|,|b_2|=\tilde\Theta(1)$ and $d^{C_b}|a_{3-\ge(g_*)}|=\bigOtilde(1)$. In addition, $|b_1|=\tilde\Theta(1)$ if $\ge(g_*)=1$. If $g_*$ is even, we have $a_1=b_1=0$ by symmetry.
}

\subsection{Output of Updated Mamba Layer}
Lastly, we characterize the output of the Mamba layer with the updated parameter $\vgamma^*$, which serves as the input to the MLP layer. This characterization is given in the following proposition, which is a formal statement of Proposition~\ref{proposition:test-time_feature_learning}.
{
\begin{proposition}\label{proposition:updated_mamba}
Let $(\vx_1,y_1,\dots,\vx_N,y_N,\vx)$ be a prompt from the task distribution with $\vbeta\sim\mathrm{Unif}(S_r)$ and $N\leq N^*$. Suppose \eqref{eq:updated_mamba} holds,
$N=\tilde\Omega(r^{3\ge(g_*)})$, and
$\eta=\Theta(d^{2C_b}(\log d)^{-C_\eta})$ with sufficiently large $C_\eta$.
Let $q=\ge(g_*)$, $k=\deg(g_*)$, and $s=\langle\vbeta,\vx\rangle$.
Under the hyperparameter conditions of Theorem~\ref{thm:main}, the following holds with high probability:
\begin{equation}\label{eq:explicit_output_residual}
 X_N:=N^{-1}\Mamba(\mZ;\vgamma^*)[\tilde d+1,N+1]
 =P_1+P_2(s/r)^q+Q(\vbeta,\vx)+R_{\mathrm{mix}}(\vbeta,\vx)+\epsilon_N,
\end{equation}
where
\begin{align*}
 P_1&=2\eta\left(a_0^2b_0-\frac{a_2^2b_2}{2r(r+2)}\right),
 &B_2&=\frac{\eta r}{r+2}a_2^2b_2,\\
 P_2&=\begin{cases}2\eta a_1^2b_1,&q=1,\\B_2,&q=2,\end{cases}
 &Q(\vbeta,\vx)&=\frac{B_2}{2r^2}\sum_{i\in\gI}\vbeta[i]^2(\vx[i]^2-1),\\
 R_{\mathrm{mix}}(\vbeta,\vx)&=\mathbbm{1}[q=1]B_2(s/r)^2.
\end{align*}
These coefficients are independent of the prompt data,
$P_1=o(1)$, and $|P_2|=\Theta((\log d)^{-C_{P_2}})$ for some $C_{P_2}>0$.
The error satisfies
\begin{equation}\label{eq:context_error_bound}
 |\epsilon_N|\leq\bigOtilde\!\left(
 N^{-1/2}+d^{-C_b+6}+d^{-C_b+C^*}
 +d^4N_\pt^{-1/2}+d^3T_1^{-1/2}\right).
\end{equation}
For $q=1$, we choose a fixed $L_{\mathrm{mix}}\geq0$ such that
$|B_2/P_2|=O((\log d)^{L_{\mathrm{mix}}})$; for $q=2$, we set $L_{\mathrm{mix}}=0$.
A sufficient growth condition is

\begin{equation}\label{eq:growth_exponent}
 r\geq(\log d)^{C_r},\qquad
 C_r\geq\max\{4k+2,L_{\mathrm{mix}}+2k+2\}.
\end{equation}

With sufficiently large logarithmic factors in the sample sizes and sufficiently large $C_b$, we have
\begin{equation}\label{eq:residual_absorption}
 |Q+R_{\mathrm{mix}}+\epsilon_N|
 =o(|P_2|r^{-q}(\log d)^{-2k})
\end{equation}
with high probability. {For even links, $R_{\mathrm{mix}}=0$ and $C_r\geq4k+2$ suffices.}
\end{proposition}
}

{We fix $C^*$ first and successively choose $C_b,C_\rho,C_\gamma,C_\eta,C_r$ sufficiently large. The logarithmic factors in the sample sizes and the high-probability constants are chosen last.}

{
\begin{proof}[Proof of Proposition~\ref{proposition:updated_mamba}]
Let
$\vc=\mathbb E_{\vbeta'\sim\mathrm{Unif}(S_r)}
[\psi(\vbeta',a_0,a_1,a_2)\odot\psi(\vbeta',b_0,b_1,b_2)]$.
By the Cauchy--Schwarz inequality and \eqref{eq:v_norm}, we have
\begin{equation*}
 \|\vc\|_1\leq\mathbb E_{\vbeta'}
 [\|\psi(\vbeta',a_0,a_1,a_2)\|
  \|\psi(\vbeta',b_0,b_1,b_2)\|]
 =\bigOtilde(d^{-C_b}).
\end{equation*}
From \eqref{eq:updated_mamba}, we can write $\vgamma^*=2\eta\vc+\mathbf e_\gamma$ with
\begin{equation*}
 \|\mathbf e_\gamma\|
 =\eta\bigOtilde(d^{-2C_b+4}
       +d^{-C_b+2}N_\pt^{-1/2}
       +d^{-C_b+1}T_1^{-1/2}).
\end{equation*}
On the Gaussian high-probability event,
$|y_j|=\bigOtilde(1)$, all gates are $O(d^{-C_b})$, and
$\|\phi(\vx_j)\|,\|\phi(\vx)\|=\bigOtilde(d)$.
Thus, the contribution of $\mathbf e_\gamma$ to $X_N$ is bounded by
\begin{equation*}
 \bigOtilde(d^{-C_b+2})\|\mathbf e_\gamma\|
 =\bigOtilde(d^{-C_b+6}+d^4N_\pt^{-1/2}+d^3T_1^{-1/2}).
\end{equation*}
For the contribution of $2\eta\vc$, the error from replacing $G_{j,N+1}$ by the single gate $\sigmoid(y_j/\rho+b)$ is bounded by
\begin{equation*}
 2\eta\,\bigOtilde(d^{-2C_b})(N+1)d^{-C_b}
 =\bigOtilde(d^{-C_b+C^*}).
\end{equation*}
Here, we used $N\leq d^{C^*}$,
$\|\phi(\vx_j)\|_\infty,\|\phi(\vx)\|_\infty=\bigOtilde(1)$, and the bound on $\|\vc\|_1$.

We use the same truncation as in Lemma~\ref{lemma:vector_concentration}.
Conditional on $\vbeta,\vx$, the variables
\begin{equation*}
 V_j=2\eta\Bigl(\rho\bar g_*\bigl(\langle\vbeta,\vx_j\rangle\bigr)+\zeta_j\Bigr)\sigmoid\Bigl(b+\bar g_*\bigl(\langle\vbeta,\vx_j\rangle\bigr)+\zeta_j/\rho\Bigr)\phi(\vx_j)^\top\bigl(\vc\odot\phi(\vx)\bigr)
\end{equation*}
are independent and equal
$2\eta y_j\sigmoid(y_j/\rho+b)\phi(\vx_j)^\top(\vc\odot\phi(\vx))$
simultaneously with high probability.
For sufficiently large $C_\eta$, $|V_j|\leq1$ with high probability.
The truncated gate is uniformly $O(d^{-C_b})$, and the magnitude of the truncated label is at most
$|g_*(\langle\vbeta,\vx_j\rangle)|+\tau$.
Since each embedding coordinate has bounded fixed moments and
$2\eta d^{-C_b}\|\vc\|_1\|\phi(\vx)\|_\infty=\bigOtilde(1)$,
the conditional second moments of $V_j$ are $\bigOtilde(1)$ on the query event.

We apply H\"{o}effding's inequality to $V_j\mathbbm{1}[|V_j|\leq1]$.
By Cauchy--Schwarz, the difference of its expectation from $\mathbb E[V_j\mid\vbeta,\vx]$ is bounded by
\begin{equation*}
 \big(\mathbb E\left[V_j^2\mid\vbeta,\vx\right]\big)^{1/2}
 \mathbb P\big[|V_j|>1\mid\vbeta,\vx\big]^{1/2},
\end{equation*}
which is negligible compared with $N^{-1/2}$ for sufficiently large high-probability constants.
By the expectation calculation in Lemma~\ref{lemma:vector_concentration}, we obtain
\begin{equation}\label{eq:query-depend}
 X_N=2\eta\psi(\vbeta,a_0,a_1,a_2)^\top(\vc\odot\phi(\vx))+\epsilon_N,
\end{equation}
with \eqref{eq:context_error_bound}.

We now compute the population term. For distinct $i,j\in\gI$,
\begin{equation*}
 \mathbb E[\vbeta'[i]^2]=\frac1r,\qquad
 \mathbb E[\vbeta'[i]^4]=\frac3{r(r+2)},\qquad
 \mathbb E[\vbeta'[i]^2\vbeta'[j]^2]=\frac1{r(r+2)}.
\end{equation*}
Coordinates outside $\gI$ vanish. By substituting these moments into the coordinatewise product, we have
\begin{align*}
 &\quad \psi(\vbeta,a_0,a_1,a_2)^\top(\vc\odot\phi(\vx))
 \\
 &=a_0^2b_0+\frac{a_1^2b_1}{r}s
 +\frac{3a_2^2b_2}{4r(r+2)}
   \sum_{i\in\gI}\vbeta[i]^2(\vx[i]^2-1) +\frac{a_2^2b_2}{r(r+2)}
   \sum_{\substack{i,j\in\gI\\i<j}}\vbeta[i]\vbeta[j]\vx[i]\vx[j]\\
 &=a_0^2b_0-\frac{a_2^2b_2}{2r(r+2)}
   +a_1^2b_1\frac{s}{r}
   +\frac{a_2^2b_2r}{2(r+2)}(s/r)^2+\frac{a_2^2b_2}{4r(r+2)}
    \sum_{i\in\gI}\vbeta[i]^2(\vx[i]^2-1).
\end{align*}
Multiplying by $2\eta$ gives \eqref{eq:explicit_output_residual}.
For even links, $a_1=b_1=0$. For non-even links, both degrees can occur.
Using the coefficient magnitude bounds with a sufficiently large $C_\eta$, we obtain
$P_1=o(1)$ and the stated inverse-polylogarithmic magnitude of $P_2$.

It remains to bound the residuals. Write the nonzero coordinates of $\vbeta$ as $\mathbf g/\|\mathbf g\|$, where $\mathbf g\sim\gN(0,\mI_r)$. Gaussian concentration for the $1$-Lipschitz function $\|\mathbf g\|_4$, together with
$\mathbb E[\|\mathbf g\|_4]\leq(3r)^{1/4}$ and the chi-square lower-tail bound, gives
\begin{equation*}
 \sum_{i\in\gI}\vbeta[i]^4=O(r^{-1}),\qquad
 \max_{i\in\gI}\vbeta[i]^2=O((\log d)/r)
\end{equation*}
{with high probability under \eqref{eq:growth_exponent}, since $k\geq1$ implies $C_r\geq6$;} the second bound follows by a coordinatewise Gaussian tail bound.
Conditional on such a direction, we apply Bernstein's inequality to obtain
\begin{align*}
 \left|\sum_{i\in\gI}\vbeta[i]^2(\vx[i]^2-1)\right|
 &\leq O\left(\sqrt{(\log d)\sum_i\vbeta[i]^4}
       +(\log d)\max_i\vbeta[i]^2\right)\\
 &=O\left(\sqrt{\frac{\log d}{r}}+\frac{(\log d)^2}{r}\right)
  =O\left(\sqrt{\frac{\log d}{r}}\right).
\end{align*}
The last equality follows from \eqref{eq:growth_exponent}. Thus,
$|Q|=O(|B_2|r^{-5/2}(\log d)^{1/2})$.

For $q=2$, $B_2=P_2$ and $R_{\mathrm{mix}}=0$.
{
The condition $C_r\geq4k+2$ implies
$r^{-1/2}(\log d)^{2k+1/2}\leq(\log d)^{-1/2}=o(1)$, so
$|Q|=o(|P_2|r^{-2}(\log d)^{-2k})$.
}
For $q=1$, $s^2=O(\log d)$ with high probability, and hence
\begin{equation*}
 |R_{\mathrm{mix}}|=O(|P_2|r^{-2}(\log d)^{L_{\mathrm{mix}}+1}),\qquad
 |Q|=O(|P_2|r^{-5/2}(\log d)^{L_{\mathrm{mix}}+1/2}).
\end{equation*}

Both are $o(|P_2|r^{-1}(\log d)^{-2k})$ when
$C_r\geq L_{\mathrm{mix}}+2k+2$, since
\begin{align*}
 r^{-1}(\log d)^{L_{\mathrm{mix}}+2k+1}
 &\leq(\log d)^{-1}=o(1),\\
 r^{-3/2}(\log d)^{L_{\mathrm{mix}}+2k+1/2}
 &\leq(\log d)^{-(C_r+3)/2}=o(1).
\end{align*}

Finally, \eqref{eq:context_error_bound} is
$o(|P_2|r^{-q}(\log d)^{-2k})$ under the sample-size assumptions with sufficiently large logarithmic factors. The terms containing $d^{-C_b}$ are absorbed when $C_b$ is sufficiently large relative to $C^*$ and the fixed exponents. This proves \eqref{eq:residual_absorption}.
\end{proof}
}

\section{Optimizing MLP Layer}\label{appendix:mlp}
In this section, we analyze the second stage of pretraining, which focuses on the MLP layer. All probability statements below are conditional on the Stage-I event on which \eqref{eq:updated_mamba} holds.

\subsection{Construction of Approximating MLP Layer}
First, we construct an infinite-width MLP approximating the link function $g_*$.
\begin{lemma}\label{lemma:infinite_width}
Assume \eqref{eq:growth_exponent}. For a given $\vbeta\in\R^d$ with $\|\vbeta\|=1$, suppose
\begin{equation*}
 h(\vx)=P_1+P_2\left(\frac{\langle\vbeta,\vx\rangle}{r}\right)^{\ge(g_*)}+n(\vx),
\end{equation*}
where $P_1=o(1)$, $|P_2|=\Theta((\log d)^{-C_{P_2}})$, and
$|n(\vx)|=o(|P_2|r^{-\ge(g_*)}(\log d)^{-2\deg(g_*)+2})$ with high probability.
Then there exists a function $\pi:\{\pm1\}\times[-1,1]\to\R$, depending on $g_*,P_1,P_2,r$ but not on $\vbeta$, such that
\begin{equation*}
 \left|\mathbb E_{v\sim\mathrm{Unif}(\{\pm1\}),\,a\sim\mathrm{Unif}([-1,1])}
 [\pi(v,a)\relu(vh(\vx)+a)]-g_*(\langle\vbeta,\vx\rangle)\right|=o(1)
\end{equation*}
with high probability, and $\sup_{v,a}|\pi(v,a)|=\bigOtilde(r^{2\ge(g_*)})$.
The conclusion also applies when $n$ depends on the context, with the probability statement including that randomness.
\end{lemma}

\begin{proof}[Proof of Lemma~\ref{lemma:infinite_width}]

By Lemma~\ref{lemma:ge}, there exists a polynomial $\tilde g_*$ such that
$g_*(z)=\tilde g_*(\operatorname{sign}(P_2)z^{\ge(g_*)})$.
Let $\tilde g_*(z)=\sum_{k=0}^{\deg(\tilde g_*)}s_kz^k$ and
$p=|P_2|r^{-\ge(g_*)}>0$.
By Lemma 17 in \citet{damian2022neural}, for each $k$ there exists a bounded function $\pi_k'$ such that
\begin{equation*}
 \mathbb E_{v,a}[\pi_k'(v,a)\relu(vz+a)]=z^k
 \quad (|z|\leq1),\qquad
 \sup_{v,a}|\pi_k'(v,a)|=O(1),
\end{equation*}
where $v$ and $a$ are uniform on $\{\pm1\}$ and $[-1,1]$, respectively.
Define
\begin{equation*}
 \pi'(v,a)=\sum_{k=0}^{\deg(\tilde g_*)}
 s_k\frac{(\log d)^{2k}}{p(\log d)^2}
 \pi_k'\left(v,\frac{a}{p(\log d)^2}\right)
 \qquad (|a|\leq p(\log d)^2).
\end{equation*}
By positive homogeneity of ReLU and a change of variables, for $|z|\leq(\log d)^2$ we have
\begin{align*}
 &\quad \mathbb E_{\substack{v\sim\mathrm{Unif}(\{\pm1\})\\
 a\sim\mathrm{Unif}([-p(\log d)^2,p(\log d)^2])}}
 [\pi'(v,a)\relu(vpz+a)]\\
 &=\sum_{k=0}^{\deg(\tilde g_*)}s_k(\log d)^{2k}
 \mathbb E_{v,a}\left[\pi_k'(v,a)\relu\left(\frac{vz}{(\log d)^2}+a\right)\right]
 =\tilde g_*(z).
\end{align*}
Moreover, $\sup_{v,a}|\pi'(v,a)|=O(p^{-1}(\log d)^{2\deg(g_*)-2})$.

To account for the offset and the original bias distribution, define
\begin{equation}\label{eq:pi}
 \pi(v,a)=
 \frac{\mathbbm{1}[|a+vP_1|\leq p(\log d)^2]}{p(\log d)^2}\,
 \pi'(v,a+vP_1).
\end{equation}
For large $d$, $|P_1|+p(\log d)^2<1$ by \eqref{eq:growth_exponent}.
For each $v$, the selected interval has probability $p(\log d)^2$ under the uniform bias distribution, which accounts for the normalization in \eqref{eq:pi}.
Consequently,
\begin{align*}
 &\quad\mathbb E_{v,a}\left[\pi(v,a)\relu\left(v(P_1+pz)+a\right)\right]\\
 &=\mathbb E_{v\sim\mathrm{Unif}(\{\pm1\})
 a\sim\mathrm{Unif}([-p(\log d)^2,p(\log d)^2])}
 [\pi'(v,a)\relu(vpz+a)]
 =\tilde g_*(z).
\end{align*}
With high probability, $|\langle\vbeta,\vx\rangle|^{\ge(g_*)}\leq(\log d)^2$.
Substituting $z=\operatorname{sign}(P_2)\langle\vbeta,\vx\rangle^{\ge(g_*)}$ gives the target $g_*(\langle\vbeta,\vx\rangle)$.
By the Lipschitz continuity of ReLU, adding $n(\vx)$ changes this expectation by at most
\begin{equation*}
 \sup_{v,a}|\pi'(v,a)|\,|n(\vx)|
 =O(p^{-1}(\log d)^{2\deg(g_*)-2})\,|n(\vx)|=o(1).
\end{equation*}
Finally, we have
\begin{equation*}
 \sup_{v,a}|\pi(v,a)|
 \leq\frac{\sup_{v,a}|\pi'(v,a)|}{p(\log d)^2}
 =\bigOtilde(r^{2\ge(g_*)}).
\end{equation*}
Thus, the construction applies to every task direction and either sign of $P_2$.

\end{proof}

Next, we approximate the link function with a finite-width MLP.
\begin{lemma}\label{lemma:finite_width}
Let $(\vv[j],\va[j])_{j=1}^m$ be independent and uniform on $\{\pm1\}\times[-1,1]$.
Under the conditions of Lemma~\ref{lemma:infinite_width} and
$m=\tilde\Omega(r^{4\ge(g_*)})$ with sufficiently large logarithmic factors, there exists a common vector $\vu'\in\R^m$ such that
\begin{equation*}
 \left|\sum_{j=1}^m\vu'[j]\relu(\vv[j]h(\vx)+\va[j])
 -g_*(\langle\vbeta,\vx\rangle)\right|=o(1),\qquad
 \|\vu'\|^2=\bigOtilde(r^{3\ge(g_*)}m^{-1})
\end{equation*}
with high probability. The approximation holds simultaneously on polynomially many independent prompts.
\end{lemma}

\begin{proof}[Proof of Lemma~\ref{lemma:finite_width}]

We choose $\vu'[j]=\pi(\vv[j],\va[j])/m$ with $\pi$ from \eqref{eq:pi}.
Since $|h(\vx)|\leq1$ with high probability, each activation is at most $2$.
By Hoeffding's inequality conditional on a prompt,
\begin{align*}
 \sum_{j=1}^m\vu'[j]\relu(\vv[j]h(\vx)+\va[j])
 &=\mathbb E_{v,a}[\pi(v,a)\relu(vh(\vx)+a)]
   +\bigOtilde(r^{2\ge(g_*)}m^{-1/2})\\
 &=g_*(\langle\vbeta,\vx\rangle)+o(1).
\end{align*}
A union bound gives the approximation simultaneously on polynomially many prompts, even though they share the same random features.

By \eqref{eq:pi}, $\pi(v,a)$ is nonzero with probability $\bigOtilde(r^{-\ge(g_*)})$.
Combining this with $\sup_{v,a}|\pi(v,a)|=\bigOtilde(r^{2\ge(g_*)})$ gives
$\mathbb E_{v,a}[\pi(v,a)^2]=\bigOtilde(r^{3\ge(g_*)})$.
Another application of Hoeffding's inequality yields
\begin{align*}
 \|\vu'\|^2
 &=m^{-2}\sum_{j=1}^m\pi(\vv[j],\va[j])^2\\
 &\leq\bigOtilde(r^{3\ge(g_*)}m^{-1}+r^{4\ge(g_*)}m^{-3/2})
 =\bigOtilde(r^{3\ge(g_*)}m^{-1}).
\end{align*}
This vector is independent of the Stage-II query-label noises.

\end{proof}

\subsection{Characterization of Estimation Error on the Training Set}
The following lemma characterizes the training error after pretraining.
\begin{lemma}\label{lemma:train_error}
There is a deterministic choice $\lambda_2>0$ such that, with probability at least $0.999$,
\begin{equation*}
 \frac1{T_2}\sum_{t=T_1+1}^{T_1+T_2}
 |y^t-f(\mZ^t;\vgamma^*,\vu^*,\vv^*,\va^*)|
 =\tau+o(1),\qquad
 \|\vu^*\|=\bigOtilde(r^{3\ge(g_*)/2}m^{-1/2}).
\end{equation*}
\end{lemma}

\begin{proof}[Proof of Lemma~\ref{lemma:train_error}]

By \eqref{eq:explicit_output_residual}--\eqref{eq:residual_absorption}, the Mamba output satisfies the conditions of Lemma~\ref{lemma:infinite_width}.
By Lemma~\ref{lemma:finite_width} and a union bound over the Stage-II prompts, there exists $\vu'$ such that
\begin{align*}
 \max_{T_1<t\leq T_1+T_2}
 |f(\mZ^t;\vgamma^*,\vu',\vv^*,\va^*)-g_*(\langle\vbeta^t,\vx^t\rangle)|
 &=o(1),\\
 \|\vu'\|^2&=\bigOtilde(r^{3\ge(g_*)}m^{-1})
\end{align*}
on an event of probability at least $1-5\cdot10^{-4}$.
On this event, the normalized Mamba outputs have magnitude at most $1$, so the MLP activations are at most $2$.

We condition on Stage I, the random features, and the Stage-II prompts except their query-label noises.
The noises $\zeta^t=y^t-g_*(\langle\vbeta^t,\vx^t\rangle)$ are then independent, centered, and uniform on $\{\pm\tau\}$.
Using this independence and the activation bound gives
\begin{align*}
 &\quad \mathbb E\left[
 \sum_{j=1}^m\left(
 \frac1{T_2}\sum_{t=T_1+1}^{T_1+T_2}
 \zeta^t\relu\left(\vv^*[j]N_\pt^{-1}\Mamba(\mZ^t;\vgamma^*)+\va^*[j]\right)
 \right)^2\right]\\
 &\leq\frac{4\tau^2m}{T_2},
\end{align*}
where the expectation is over the remaining noises.
Here and below, $\Mamba(\mZ^t;\vgamma^*)$ denotes the scalar output at $[\tilde d+1,N_\pt+1]$, as in the model definition.
By Markov's inequality and the Cauchy--Schwarz inequality, with conditional probability at least $1-5\cdot10^{-4}$,
\begin{equation*}
 \left|\frac1{T_2}\sum_{t=T_1+1}^{T_1+T_2}
 \zeta^t f(\mZ^t;\vgamma^*,\vu,\vv^*,\va^*)\right|
 \leq O\left(\sqrt{\frac m{T_2}}\right)\|\vu\|
 \qquad\text{for every }\vu\in\R^m.
\end{equation*}

Choose a deterministic regularization coefficient
\begin{equation}\label{eq:ridge_choice}
 \lambda_2=mr^{-3\ge(g_*)}(\log d)^{-C}
\end{equation}
for a sufficiently large fixed constant $C>0$.
The norm bound on $\vu'$ gives $\lambda_2\|\vu'\|^2=o(1)$.
With sufficiently large logarithmic factors in the assumed lower bound on $T_2$, we also have
$m/(\lambda_2T_2)=o(1)$.
By optimality of $\vu^*$ for the empirical squared loss plus $\lambda_2\|\vu\|^2/2$, expanding around the noiseless targets gives
\begin{align*}
 &\quad \frac1{T_2}\sum_{t=T_1+1}^{T_1+T_2}
 \left(f(\mZ^t;\vgamma^*,\vu^*,\vv^*,\va^*)-g_*(\langle\vbeta^t,\vx^t\rangle)\right)^2
 +\frac{\lambda_2}{2}\|\vu^*\|^2\\
 &\leq o(1)+\frac{\lambda_2}{2}\|\vu'\|^2
 +\frac2{T_2}\sum_{t=T_1+1}^{T_1+T_2}
 \zeta^t f(\mZ^t;\vgamma^*,\vu^*-\vu',\vv^*,\va^*)\\
 &\leq o(1)+\frac{\lambda_2}{2}\|\vu'\|^2
 +O\left(\sqrt{\frac m{T_2}}\right)(\|\vu^*\|+\|\vu'\|)\\
 &\leq o(1)+\lambda_2\|\vu'\|^2
 +\frac{\lambda_2}{4}\|\vu^*\|^2
 +O\left(\frac m{\lambda_2T_2}\right),
\end{align*}
where the last step uses Young's inequality.
Therefore,
\begin{equation*}
 \frac1{T_2}\sum_{t=T_1+1}^{T_1+T_2}
 \left(f(\mZ^t;\vgamma^*,\vu^*,\vv^*,\va^*)-g_*(\langle\vbeta^t,\vx^t\rangle)\right)^2
 +\frac{\lambda_2}{4}\|\vu^*\|^2=o(1).
\end{equation*}
In particular, $\|\vu^*\|^2=o(\lambda_2^{-1})=\bigOtilde(r^{3\ge(g_*)}m^{-1})$.
Finally, since $|\zeta^t|=\tau$, the triangle and Cauchy--Schwarz inequalities imply
\begin{align*}
 &\quad \left|\frac1{T_2}\sum_{t=T_1+1}^{T_1+T_2}
 |y^t-f(\mZ^t;\vgamma^*,\vu^*,\vv^*,\va^*)|-\tau\right|\\
 &\leq\frac1{T_2}\sum_{t=T_1+1}^{T_1+T_2}
 |f(\mZ^t;\vgamma^*,\vu^*,\vv^*,\va^*)-g_*(\langle\vbeta^t,\vx^t\rangle)|
 =o(1).
\end{align*}
A union bound gives the desired conclusions with probability at least $0.999$.

\end{proof}

\section{Test Error Analysis}\label{appendix:test_error}
We condition throughout this section on the Stage-I event on which \eqref{eq:updated_mamba} holds. Its failure probability is included in the final union bound. We analyze the test-time estimation error
\begin{equation*}
 \gR_{N_{\mathrm{test}}}(\vgamma^*,\vu^*,\vv^*,\va^*)
 =\mathbb E_{(\mZ,y)\sim\gD(N_{\mathrm{test}})}
 [|f(\mZ;\vgamma^*,\vu^*,\vv^*,\va^*)-y|].
\end{equation*}
As in the model definition, $\Mamba(\mZ;\vgamma^*)$ denotes the scalar output at $[\tilde d+1,N+1]$ for a prompt of context length $N$.

\subsection{Test Error for Prompts with Pretraining Context Length}
We first prove the conclusion for $N_{\mathrm{test}}=N_\pt$ by establishing a generalization bound using Rademacher complexity.

We define a family of functions on inputs with context length $N_\pt$:
\begin{equation*}
 \gF_U=\left\{(\mZ,y)\mapsto\sum_{j=1}^m\vu[j]\relu\left(
 \vv^*[j]N_\pt^{-1}\Mamba(\mZ;\vgamma^*)+\va^*[j]\right)
 \ \middle|\ \|\vu\|\leq U\right\}.
\end{equation*}
Its Rademacher complexity for sample size $T_2$ is
\begin{equation*}
 \mathrm{Rad}_{T_2}(\gF_U)
 =\mathbb E_{\substack{(\mZ^t,y^t)\sim\gD(N_\pt)\\
 \bm\epsilon\sim\mathrm{Unif}(\{\pm1\}^{T_2})}}
 \left[\sup_{\tilde f\in\gF_U}\frac1{T_2}
 \sum_{t=1}^{T_2}\bm\epsilon[t]\tilde f(\mZ^t,y^t)\right].
\end{equation*}

\begin{lemma}\label{lemma:rademacher}
It holds that
\begin{equation*}
 \mathrm{Rad}_{T_2}(\gF_U)=\bigOtilde(Um^{1/2}T_2^{-1/2}).
\end{equation*}
\end{lemma}

\begin{proof}[Proof of Lemma~\ref{lemma:rademacher}]

We condition on the Stage-I data and Stage-II initialization satisfying the preceding bounds. These parameters are independent of the Stage-II training prompts.
By the Cauchy--Schwarz and Jensen inequalities and the independence of the Rademacher signs,
\begin{align*}
 \mathrm{Rad}_{T_2}(\gF_U)
 &\leq\frac U{T_2}\left(
 \sum_{t=1}^{T_2}\sum_{j=1}^m
 \mathbb E\left[\relu\left(\vv^*[j]N_\pt^{-1}
 \Mamba(\mZ^t;\vgamma^*)+\va^*[j]\right)^2\right]\right)^{1/2}\\
 &\leq U\left(\frac{2m}{T_2}
 \left(1+\mathbb E_{\mZ}\left[N_\pt^{-2}\Mamba(\mZ;\vgamma^*)^2\right]\right)\right)^{1/2},
\end{align*}
since $|\vv^*[j]|=1$ and $|\va^*[j]|\leq1$.

We bound the second moment for every context length satisfying Proposition~\ref{proposition:updated_mamba}.
By Lemma~\ref{lemma:sigmoid_ie}, $a_i=\bigOtilde(d^{-C_b})$ for $i=0,1,2$, while $b_i=\bigOtilde(1)$.
Thus, \eqref{eq:updated_mamba}, our choice of $\eta$, and the sample-size bounds give
$\|\vgamma^*\|_\infty=\bigOtilde(d^{C_b})$.
Since the gating factors lie in $[0,1]$,
\begin{equation*}
 |N^{-1}\Mamba(\mZ;\vgamma^*)|
 \leq\|\vgamma^*\|_\infty\|\phi(\vx)\|
 \frac1N\sum_{j=1}^N|y_j|\|\phi(\vx_j)\|.
\end{equation*}
Each of the $O(d^2)$ embedding coordinates is a fixed-degree Hermite polynomial.
Consequently,
\begin{equation*}
 \mathbb E[\|\phi(\vx)\|^4]=O(d^4),\qquad
 \mathbb E[\|\phi(\vx)\|^8]=O(d^8),\qquad
 \sup_{\vbeta\in S_r}\mathbb E[|y_1|^8\mid\vbeta]=O(1).
\end{equation*}
Conditional on $\vbeta$, the query is independent of the context.
Applying Jensen's inequality to the context average and then the Cauchy--Schwarz inequality gives
\begin{align}\label{eq:fourth_moment_envelope}
 \mathbb E_{\mZ\sim\gD(N)}[N^{-4}\Mamba(\mZ;\vgamma^*)^4]
 &\leq\|\vgamma^*\|_\infty^4\mathbb E[\|\phi(\vx)\|^4]
 \sup_{\vbeta\in S_r}\mathbb E[|y_1|^4\|\phi(\vx_1)\|^4\mid\vbeta]\nonumber\\
 &\leq\|\vgamma^*\|_\infty^4 O(d^4)
 \sup_{\vbeta\in S_r}\big(\mathbb E[|y_1|^8\mid\vbeta]\big)^{1/2}
 \big(\mathbb E[\|\phi(\vx_1)\|^8]\big)^{1/2}\nonumber\\
 &=\bigOtilde(d^{4C_b+8}).
\end{align}
Let $\mathcal E$ be the event that the conclusion of Proposition~\ref{proposition:updated_mamba} holds for this context length.
On $\mathcal E$, $|N^{-1}\Mamba(\mZ;\vgamma^*)|=\bigOtilde(1)$.
Therefore,
\begin{align}\label{eq:second_moment_uniform}
 &\quad \mathbb E_{\mZ\sim\gD(N)}[N^{-2}\Mamba(\mZ;\vgamma^*)^2]\nonumber\\
 &\leq\bigOtilde(1)
 +\left(\mathbb E_{\mZ\sim\gD(N)}[N^{-4}\Mamba(\mZ;\vgamma^*)^4]\right)^{1/2}
 \mathbb P[\mathcal E^c]^{1/2}\nonumber\\
 &\leq\bigOtilde(1)+\bigOtilde(d^{2C_b+4-C_{\mathrm{whp}}/2})
 =\bigOtilde(1)
\end{align}
for $C_{\mathrm{whp}}>4C_b+8$.
This bound is uniform over the admissible context lengths.
Taking $N=N_\pt$ in the Rademacher bound proves the lemma.

\end{proof}

Next, we obtain the following result on test error.
\begin{lemma}\label{lemma:test_error_same}
With probability at least $0.995$, it holds that
$\gR_{N_\pt}(\vgamma^*,\vu^*,\vv^*,\va^*)-\tau=o(1)$.
\end{lemma}

\begin{proof}[Proof of Lemma~\ref{lemma:test_error_same}]

By Lemma~\ref{lemma:train_error}, there is a deterministic
$U=\bigOtilde(r^{3\ge(g_*)/2}m^{-1/2})$ such that, with probability at least $0.999$,
the trained predictor belongs to $\gF_U$ and its empirical absolute loss is $\tau+o(1)$.

The map $z\mapsto|y-z|-|y|$ is $1$-Lipschitz and vanishes at zero.
Applying symmetrization (Proposition 4.2 in \citet{bach2024learning}) and the Ledoux--Talagrand contraction inequality to these centered losses gives
\begin{align*}
 &\quad \mathbb E\left[\sup_{\tilde f\in\gF_U}\left|
 \mathbb E_{(\mZ,y)\sim\gD(N_\pt)}[|y-\tilde f(\mZ,y)|-|y|]
 -\frac1{T_2}\sum_{t=T_1+1}^{T_1+T_2}
 (|y^t-\tilde f(\mZ^t,y^t)|-|y^t|)
 \right|\right]\\
 &\leq4\,\mathrm{Rad}_{T_2}(\gF_U).
\end{align*}
Here, the outer expectation is over the independent Stage-II training prompts.
The remaining label-only term satisfies
\begin{equation*}
 \mathbb E\left[\left|\mathbb E[|y|]
 -\frac1{T_2}\sum_{t=T_1+1}^{T_1+T_2}|y^t|\right|\right]
 \leq\sqrt{\frac{\operatorname{Var}(|y|)}{T_2}}
 \leq\sqrt{\frac{\mathbb E[y^2]}{T_2}}=O(T_2^{-1/2}).
\end{equation*}
Combining these bounds with Lemma~\ref{lemma:rademacher},
\begin{align*}
 &\quad \mathbb E\left[\sup_{\tilde f\in\gF_U}\left|
 \mathbb E_{(\mZ,y)\sim\gD(N_\pt)}[|y-\tilde f(\mZ,y)|]
 -\frac1{T_2}\sum_{t=T_1+1}^{T_1+T_2}
 |y^t-\tilde f(\mZ^t,y^t)|
 \right|\right]\\
 &\leq4\,\mathrm{Rad}_{T_2}(\gF_U)+O(T_2^{-1/2})
 =\bigOtilde(r^{3\ge(g_*)/2}T_2^{-1/2})=o(1).
\end{align*}
By Markov's inequality, the displayed supremum is $o(1)$ with probability at least $0.996$.
Together with the event of Lemma~\ref{lemma:train_error}, this yields
\begin{equation*}
 \gR_{N_\pt}(\vgamma^*,\vu^*,\vv^*,\va^*)
 \leq\frac1{T_2}\sum_{t=T_1+1}^{T_1+T_2}
 |y^t-f(\mZ^t;\vgamma^*,\vu^*,\vv^*,\va^*)|+o(1)
 =\tau+o(1)
\end{equation*}
with probability at least $0.995$.
The independent query noise is uniform on $\{\pm\tau\}$, and
$(|a-\tau|+|a+\tau|)/2\geq\tau$ for every $a\in\R$.
Thus the risk is at least $\tau$, proving the desired equality.

\end{proof}

\subsection{Test Error for Prompts with General Length}
For the last step, we extend the result to a general test-time context length
$N_{\mathrm{test}}=\tilde\Omega(r^{3\ge(g_*)})$.

\begin{proof}[Proof of Theorem~\ref{thm:main}]

We condition on the trained parameters satisfying the preceding bounds.
We compare test prompts with the same feature vector $\vbeta$ and query-label pair $(\vx,y)$.
Their contexts consist of the first $N_\pt$ or $N_{\mathrm{test}}$ examples of a common context of length $\max\{N_\pt,N_{\mathrm{test}}\}$, and we denote the corresponding embeddings by $\mZ_{N_\pt}$ and $\mZ_{N_{\mathrm{test}}}$.

In \eqref{eq:explicit_output_residual}, all terms except the context error depend only on $\vbeta,\vx$ and the trained parameters.
They therefore cancel when we subtract the two normalized Mamba outputs.
By \eqref{eq:context_error_bound}, with high probability,
\begin{align*}
 &\quad \left|N_\pt^{-1}\Mamba(\mZ_{N_\pt};\vgamma^*)
 -N_{\mathrm{test}}^{-1}\Mamba(\mZ_{N_{\mathrm{test}}};\vgamma^*)\right|\\
 &\leq\bigOtilde\left(
 N_\pt^{-1/2}+N_{\mathrm{test}}^{-1/2}
 +d^{-C_b+6}+d^{-C_b+C^*}
 +d^4N_\pt^{-1/2}+d^3T_1^{-1/2}\right).
\end{align*}
With sufficiently large logarithmic factors in the sample sizes and sufficiently large $C_b$,
multiplying the right-hand side by
$\|\vu^*\|\sqrt m=\bigOtilde(r^{3\ge(g_*)/2})$ gives $o(1)$.
By the Lipschitz continuity of ReLU and $|\vv^*[j]|=1$, there is an event $\mathcal E$ with $\mathbb P[\mathcal E^c]\leq d^{-C_{\mathrm{whp}}}$ on which
\begin{align*}
 &\quad |f(\mZ_{N_\pt};\vgamma^*,\vu^*,\vv^*,\va^*)
 -f(\mZ_{N_{\mathrm{test}}};\vgamma^*,\vu^*,\vv^*,\va^*)|\\
 &\leq\|\vu^*\|\sqrt m
 \left|N_\pt^{-1}\Mamba(\mZ_{N_\pt};\vgamma^*)
 -N_{\mathrm{test}}^{-1}\Mamba(\mZ_{N_{\mathrm{test}}};\vgamma^*)\right|
 =o(1).
\end{align*}

To pass to the expected error, we also control the complementary event.
By the same Lipschitz bound and \eqref{eq:second_moment_uniform},
\begin{align*}
 &\quad \mathbb E\left[
 |f(\mZ_{N_\pt};\vgamma^*,\vu^*,\vv^*,\va^*)
 -f(\mZ_{N_{\mathrm{test}}};\vgamma^*,\vu^*,\vv^*,\va^*)|^2\right]\\
 &\leq2m\|\vu^*\|^2\left(
 \mathbb E[N_\pt^{-2}\Mamba(\mZ_{N_\pt};\vgamma^*)^2]
 +\mathbb E[N_{\mathrm{test}}^{-2}\Mamba(\mZ_{N_{\mathrm{test}}};\vgamma^*)^2]\right)\\
 &=\bigOtilde(r^{3\ge(g_*)}).
\end{align*}
Splitting the expectation over $\mathcal E$ and $\mathcal E^c$ and applying the Cauchy--Schwarz inequality therefore yields
\begin{align*}
 &\quad \mathbb E\left[
 |f(\mZ_{N_\pt};\vgamma^*,\vu^*,\vv^*,\va^*)
 -f(\mZ_{N_{\mathrm{test}}};\vgamma^*,\vu^*,\vv^*,\va^*)|\right]\\
 &\leq o(1)+\bigOtilde(r^{3\ge(g_*)/2})\mathbb P[\mathcal E^c]^{1/2}
 =o(1),
\end{align*}
where the last equality uses $r\leq d$ and $C_{\mathrm{whp}}>3\ge(g_*)$.
Finally, by the triangle inequality,
\begin{align*}
 &\quad |\gR_{N_{\mathrm{test}}}(\vgamma^*,\vu^*,\vv^*,\va^*)
 -\gR_{N_\pt}(\vgamma^*,\vu^*,\vv^*,\va^*)|\\
 &\leq\mathbb E\left[
 |f(\mZ_{N_\pt};\vgamma^*,\vu^*,\vv^*,\va^*)
 -f(\mZ_{N_{\mathrm{test}}};\vgamma^*,\vu^*,\vv^*,\va^*)|\right]
 =o(1).
\end{align*}
Combining this with Lemma~\ref{lemma:test_error_same} and taking a union bound over the preceding Stage-I and initialization events proves the conclusion with probability at least $0.99$.

\end{proof}

\newpage
\section*{NeurIPS Paper Checklist}

\begin{enumerate}

\item {\bf Claims}
    \item[] Question: Do the main claims made in the abstract and introduction accurately reflect the paper's contributions and scope?
    \item[] Answer: \answerYes{} %
    \item[] Justification: The abstract and introduction clearly state the main contributions of the paper, including limitations of prior works, a brief description of the problem setting, and a summary of our theoretical findings.

\item {\bf Limitations}
    \item[] Question: Does the paper discuss the limitations of the work performed by the authors?
    \item[] Answer: \answerYes{} %
    \item[] Justification: {We discuss the limitations of our theoretical results in Section~\ref{sec:conclusion}.}

\item {\bf Theory assumptions and proofs}
    \item[] Question: For each theoretical result, does the paper provide the full set of assumptions and a complete (and correct) proof?
    \item[] Answer: \answerYes{} %
    \item[] Justification: We introduce formal theoretical settings in Section~\ref{sec:setting} and provide full set of assumptions in the statement of Theorem~\ref{thm:main}. In addition, we provide complete proof of our theoretical findings in the appendix.

    \item {\bf Experimental result reproducibility}
    \item[] Question: Does the paper fully disclose all the information needed to reproduce the main experimental results of the paper to the extent that it affects the main claims and/or conclusions of the paper (regardless of whether the code and data are provided or not)?
    \item[] Answer: \answerYes{} %
    \item[] Justification: We provide details for our main theoretical results in Section~\ref{sec:exp}.

\item {\bf Open access to data and code}
    \item[] Question: Does the paper provide open access to the data and code, with sufficient instructions to faithfully reproduce the main experimental results, as described in supplemental material?
    \item[] Answer: \answerNo{} %
    \item[] Justification: We do not provide code in supplemental material. However, our experiments on synthetic data can be easily reproduced since we opened all details.

\item {\bf Experimental setting/details}
    \item[] Question: Does the paper specify all the training and test details (e.g., data splits, hyperparameters, how they were chosen, type of optimizer) necessary to understand the results?
    \item[] Answer: \answerYes{} %
    \item[] Justification: We provide all details for our main experimental results in Section~\ref{sec:exp}.

\item {\bf Experiment statistical significance}
    \item[] Question: Does the paper report error bars suitably and correctly defined or other appropriate information about the statistical significance of the experiments?
    \item[] Answer: \answerYes{} %
    \item[] Justification: We represent the experimental results with the mean and standard deviation in Figure~\ref{fig}.

\item {\bf Experiments compute resources}
    \item[] Question: For each experiment, does the paper provide sufficient information on the computer resources (type of compute workers, memory, time of execution) needed to reproduce the experiments?
    \item[] Answer: \answerYes{}{} %
    \item[] Justification: We provide computer resource information in Section~\ref{sec:exp}.
    
\item {\bf Code of ethics}
    \item[] Question: Does the research conducted in the paper conform, in every respect, with the NeurIPS Code of Ethics \url{https://neurips.cc/public/EthicsGuidelines}?
    \item[] Answer: \answerYes{} %
    \item[] Justification:  This paper does not violate the NeurIPS Code of Ethics in terms of data privacy, potential for misuse, or other ethical issues.

\item {\bf Broader impacts}
    \item[] Question: Does the paper discuss both potential positive societal impacts and negative societal impacts of the work performed?
    \item[] Answer: \answerNA{} %
    \item[] Justification: This work is primarily theoretical.
    
\item {\bf Safeguards}
    \item[] Question: Does the paper describe safeguards that have been put in place for responsible release of data or models that have a high risk for misuse (e.g., pre-trained language models, image generators, or scraped datasets)?
    \item[] Answer: \answerNA{} %
    \item[] Justification: This work is primarily theoretical.

\item {\bf Licenses for existing assets}
    \item[] Question: Are the creators or original owners of assets (e.g., code, data, models), used in the paper, properly credited and are the license and terms of use explicitly mentioned and properly respected?
    \item[] Answer: \answerYes{} %
    \item[] Justification: We clarify the reference of our code base in Section~\ref{sec:exp}.

\item {\bf New assets}
    \item[] Question: Are new assets introduced in the paper well documented and is the documentation provided alongside the assets?
    \item[] Answer: \answerNA{} %
    \item[] Justification: Our work does not release new assets.

\item {\bf Crowdsourcing and research with human subjects}
    \item[] Question: For crowdsourcing experiments and research with human subjects, does the paper include the full text of instructions given to participants and screenshots, if applicable, as well as details about compensation (if any)? 
    \item[] Answer: \answerNA{} %
    \item[] Justification: This research does not involve any human subjects or crowdsourcing.

\item {\bf Institutional review board (IRB) approvals or equivalent for research with human subjects}
    \item[] Question: Does the paper describe potential risks incurred by study participants, whether such risks were disclosed to the subjects, and whether Institutional Review Board (IRB) approvals (or an equivalent approval/review based on the requirements of your country or institution) were obtained?
    \item[] Answer: \answerNA{} %
    \item[] Justification: This research does not involve any human subjects. Thus, IRB approval is notapplicable.

\item {\bf Declaration of LLM usage}
    \item[] Question: Does the paper describe the usage of LLMs if it is an important, original, or non-standard component of the core methods in this research? Note that if the LLM is used only for writing, editing, or formatting purposes and does \emph{not} impact the core methodology, scientific rigor, or originality of the research, declaration is not required.
    \item[] Answer: \answerNA{} %
    \item[] Justification: We used LLMs only for writing and editing.

\end{enumerate}

\end{document}